\documentclass{article}

\PassOptionsToPackage{compress}{natbib}

\usepackage[final]{neurips_2025}

\usepackage{amsfonts}       
\usepackage{nicefrac}       
\usepackage{xcolor}         
\usepackage{natbib}
\usepackage{multirow}
\usepackage{tabularx}
\usepackage{bm}
\usepackage{utfsym}

\usepackage{microtype}
\usepackage{graphicx}
\usepackage{subfigure}
\usepackage{booktabs} 

\usepackage{hyperref}

\usepackage{amsmath}
\usepackage{amssymb}
\usepackage{mathtools}
\usepackage{amsthm}
\usepackage{mathrsfs}
\usepackage[capitalize,noabbrev]{cleveref}

\theoremstyle{plain}
\newtheorem{theorem}{Theorem}[section]
\newtheorem{proposition}[theorem]{Proposition}
\newtheorem{lemma}[theorem]{Lemma}

\theoremstyle{definition}
\newtheorem{definition}[theorem]{Definition}
\newtheorem{assumption}[theorem]{Assumption}
\theoremstyle{remark}

\newtheorem{condition}[theorem]{Condition} 

\usepackage[textsize=tiny]{todonotes}

\usepackage{xspace}

\newcommand{\bE}{\mathbb{E}}

\newcommand{\bH}{\mathbb{H}}

\newcommand{\bP}{\mathbb{P}}

\usepackage{bbding}
\usepackage{extarrows}

\usepackage[shortlabels]{enumitem}

\definecolor{myblue}{RGB}{0,0,255}

\usepackage{amsfonts}
\usepackage{amsmath}
\usepackage{bm}
\usepackage{amssymb}
\usepackage{graphicx} 
\usepackage{nccmath}
\usepackage{thmtools}
\usepackage{thm-restate}
\usepackage{hyperref}
\usepackage{cleveref}
\usepackage{color}
\usepackage{cite}
\usepackage{natbib}
\usepackage{tikz}
\usepackage{wrapfig}   
\usepackage{algorithm}
\usepackage{algorithmic}

\usepackage{blindtext}

\newcommand{\A}{{\mathcal{A}}}
\newcommand{\U}{{\mathcal{U}}}

\newcommand{\bbH}{{\mathcal{H}}}
\newcommand{\R}{{\mathcal{R}}}
\newcommand{\bbS}{{\mathbf{S}}}
\newcommand{\bbE}{{\mathcal{E}}}
\newcommand{\bone}{{\bold{1}}}
\newcommand{\N}{{\mathcal{N}}}

\newcommand{\V}{{\mathcal{V}}}
\newcommand{\bbY}{{\mathcal{Y}}}

\title{Design-Based Bandits Under Network Interference: Trade-Off Between Regret and Statistical Inference}

\author{
Zichen Wang\textsuperscript{1}\thanks{Equal contribution} \And
Haoyang Hong\textsuperscript{2}\footnotemark[1] \And
Chuanhao Li\textsuperscript{3} \And
Haoxuan Li\textsuperscript{4} \And
Zhiheng Zhang\textsuperscript{5}\thanks{Corresponding author} \And
Huazheng Wang\textsuperscript{2} \\
\\
\textsuperscript{1}Department of ECE and CSL, UIUC \\
\textsuperscript{2}School of EECS, Oregon State University \\
\textsuperscript{3}Department of Industrial Engineering, Tsinghua University \\
\textsuperscript{4}Center for Data Science, Peking University  \\
\textsuperscript{5}School of Statistics and Data Science, 
Shanghai University of Finance and Economics,\\ Shanghai 200433, P.R. China
}

\begin{document}

\maketitle

\begin{abstract}
In multi-armed bandits with network interference (MABNI), the action taken by one node can influence the rewards of others, creating complex interdependence. While existing research on MABNI largely concentrates on minimizing regret, it often overlooks the crucial concern that an excessive emphasis on the optimal arm can undermine the inference accuracy for sub-optimal arms. Although initial efforts have been made to address this trade-off in single-unit scenarios, these challenges have become more pronounced in the context of MABNI. In this paper, we establish, for the first time, a theoretical Pareto frontier characterizing the trade-off between regret minimization and inference accuracy in adversarial (design-based) MABNI. We further introduce an anytime-valid asymptotic confidence sequence along with a corresponding algorithm, $\texttt{EXP3-N-CS}$, specifically designed to balance the trade-off between regret minimization and inference accuracy in this setting.
\end{abstract}

\section{Introduction}

Network interference \citep{leung2022causal, leung2022rate, leung2023network, imbens2024causal}, a well-known concept in causal inference, describes a phenomenon where the treatment assigned to one individual can influence the outcomes of others. It has been extensively studied across various disciplines, with significant applications in economics~\citep{arpino2016assessing, munro2021treatment} and the social sciences~\citep{bandiera2009social, bond201261, paluck2016changing, imbens2024causal}. Due to its broad real-world relevance, this concept in causal inference has recently been explored and recognized by researchers in online learning. Consequently, it has begun to be frequently applied in multi-armed bandits~\citep{agarwal2024multi, jia2024multi, zhang2024online}. 

To effectively identify causal effects under network interference, a common approach involves conducting randomized experiments to estimate causal effects from experimental data~\citep{leung2022causal, leung2022rate, leung2023network, gao2023causal}. Specifically, researchers design estimators that leverage feedback collected from each individual (commonly referred to as \textit{potential outcomes} in the causal inference literature). They primarily focus on ensuring unbiasedness and controlling the variance of these estimators. However, in practice, such experiments are often conducted over multiple rounds, introducing a dynamic aspect to individual feedback. In this setting, the aforementioned \textit{potential outcomes} are also referred to as \textit{rewards} in the online learning literature, as they contribute to cumulative regret, which quantifies the overall welfare loss incurred throughout the experiment~\citep{simchi2023multi}. Once the experiment concludes, data collected in earlier rounds can be utilized to improve social welfare in future applications~\citep{mok2021managing}. For instance, when evaluating the effectiveness of different drug treatments, researchers may not only seek to maximize treatment efficacy during the trial but also estimate the relative differences in treatment effects across drugs based on experimental data. This process necessitates a careful balance between optimizing the \emph{estimation accuracy} of causal effects and minimizing the \emph{cumulative regret} incurred during the experiment~\citep{simchi2023multi, zhang2024online}. Furthermore, researchers may wish to continuously infer causal effects throughout the experiment, allowing them to make informed decisions about when to stop based on data-driven metrics or predefined thresholds~\citep{ham2023design, woong2023design, liang2023experimental}. This type of continual inference often requires \emph{anytime-validity}, ensuring that statistical inferences remain robust regardless of the time at which they are made~\citep{lindon2022anytime, waudby2024anytime}.

\begin{wrapfigure}{r}{0.45\textwidth}
\centering
\includegraphics[width=0.45\textwidth]{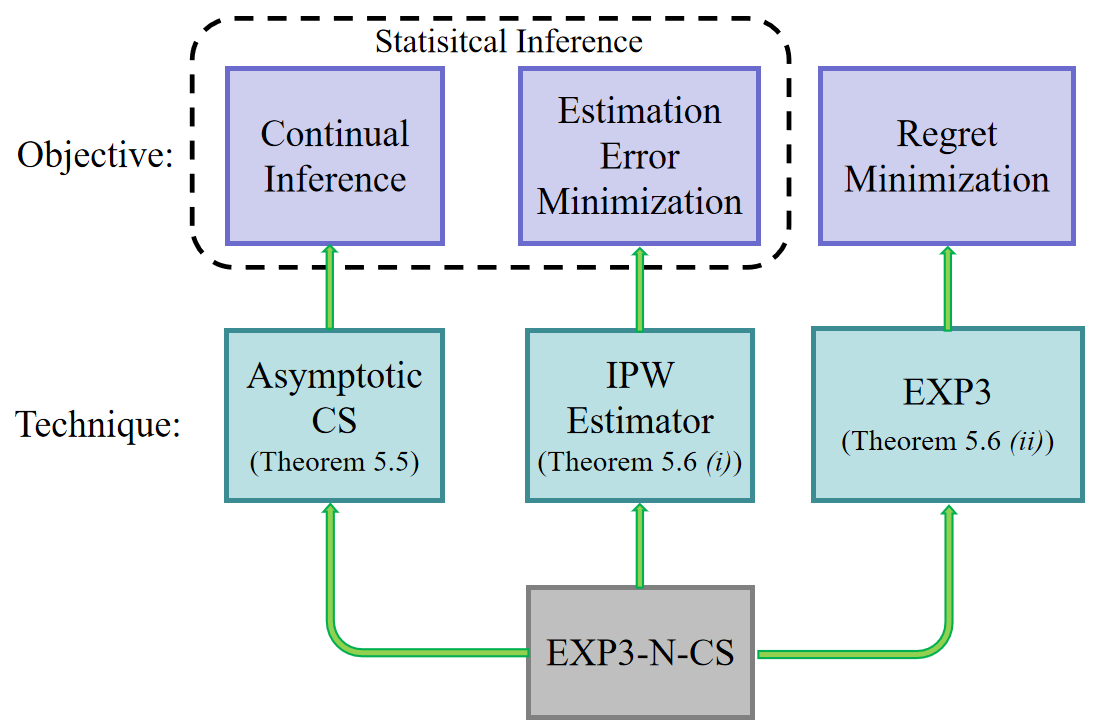}
    \caption{The main contribution of our paper is to study how to achieve these three objectives and to analyze their underlying interrelationships.}
    \label{fig:enter-label}
\end{wrapfigure}

Building on the above observations, three critical learning objectives emerge: (i) conducting continual inference on causal effects, (ii) minimizing cumulative regret, and (iii) designing estimators that leverage collected data to accurately estimate causal effects once the experiment concludes. However, most existing studies fail to address these three objectives simultaneously. For instance, \citet{jia2024multi, agarwal2024multi} focus exclusively on regret minimization in MABNI, whereas \citet{ham2023design, woong2023design} primarily explore continual inference in adversarial MAB using the technique of Asymptotic \textbf{C}onfidence \textbf{S}equences (CS)~\citep{waudby2021time}. Similarly, \citet{simchi2023multi, zhang2024online, duan2024regret} investigate the trade-off between regret minimization and causal effect estimation but place little emphasis on continual inference. The most closely related work is \citet{liang2023experimental}, which considers all three learning objectives within the framework of adversarial MAB. However, their approach suffers from two key limitations: (i) it does not account for network interference, and (ii) it lacks rigorous theoretical results characterizing the trade-off between causal effect estimation and regret minimization. Building on the above observations, we aim to achieve three key learning objectives in adversarial MABNI and make the following contributions:

\begin{itemize}
    \item We propose a unified learning framework tailored for the adversarial setting, which we term adversarial \texttt{MAB-N}. Furthermore, we establish the first Pareto frontier that delineates the fundamental trade-off between regret and causality-estimation error in adversarial \texttt{MAB-N}.
    \item We develop an anytime-valid asymptotic CS to enable continuous inference in adversarial \texttt{MAB-N}. Building on this, we introduce \texttt{EXP3-Network-Confidence Sequence} (\texttt{EXP3-N-CS}), which integrates our asymptotic CS and is specifically designed to achieve all three learning objectives (i)–(iii).
    \item We conduct simulation studies to investigate the empirical performance of our \texttt{EXP3-N-CS}.
\end{itemize}

\begin{table}[t]
\centering
\caption{The overview of the exploration of these three objectives: Obj. 1 represents regret minimization, Obj. 2 represents continual inference, and Obj. 3 corresponds to minimizing the ATE estimation error.}
\label{tab:five_columns}
\begin{tabular}{p{1.7in}cccccc}
\toprule
Paper & Obj. 1 & Obj. 2 & Obj. 3 & Trade-off & Network & Adversarial \\
\midrule
\citet{simchi2023multi}    & $\checkmark$  & $\times$   & $\checkmark$   & $\checkmark$ & $\times$ & $\times$  \\
\citet{woong2023design,ham2023design}    & $\times$   &  $\checkmark$   &  $\checkmark$  & $\times$  & $\times$  & $\checkmark$ \\
\citet{liang2023experimental}    &  $\checkmark$  & $\checkmark$ & $\checkmark$ & $\times$ & $\times$ & $\checkmark$ \\
\citet{jia2024multi}    & $\checkmark$  & $\times$   & $\times$ & $\times$ & $\checkmark$  & $\checkmark$   \\
\citet{agarwal2024multi,xu2024linear} & $\checkmark$ & $\times$ & $\times$ & $\times$ & $\checkmark$ & $\times$\\
\citet{zhang2024online}    & $\boldsymbol{\checkmark}$   & $\times$   & $\checkmark$   & $\checkmark$ & $\checkmark$ & $\times$\\ \textbf{Ours} & $\textbf{\checkmark}$ & $\textbf{\checkmark}$ & $\textbf{\checkmark}$ & $\textbf{\checkmark}$ & $\textbf{\checkmark}$ & $\textbf{\checkmark}$ \\
\bottomrule
\end{tabular}
\end{table}

The paper is organized as follows. Section~\ref{related-work} reviews related work. In Section~\ref{framework}, we introduce the core setting of adversarial MABNI, along with the techniques of exposure mapping and clustering, which support the design of \texttt{MAB-N}. Section~\ref{4} presents our results on Pareto optimality. In Section~\ref{5}, we introduce the Asymptotic CS technique and our main algorithm, \texttt{EXP3-N-CS}. Finally, Section~\ref{experiments} reports the experimental results.

\section{Related Work}
\label{related-work}

\setlength{\tabcolsep}{3pt}

\paragraph{Causality inference under network interference.}

In the current causality literature, interference is a well-established concept that signifies a violation of the Stable Unit Treatment Value Assumption (SUTVA)~\citep{imbens2024causal}. It arises in scenarios where an individual's treatment potentially influences the outcomes of others, a phenomenon frequently observed in practice. Existing research on offline causal inference under network interference has primarily employed two key methodological approaches: \emph{clustering-based methods}~\citep{zhang2023individualized, viviano2023causal, zhao2024simple} and \emph{exposure mapping techniques}~\citep{leung2022causal, leung2022rate, leung2023network,zhao2024simple}. Recently, a growing interest has been in studying the MABNI. For instance, \citet{agarwal2024multi} applied Fourier analysis to transform the MABNI problem into a sparse linear stochastic bandit formulation. However, to mitigate the exponential growth of the action space, they imposed a strong sparsity assumption on network structures, restricting the number of neighbors each node can have. In contrast, \citet{jia2024multi} explored an MABNI setting without such a sparsity assumption. Their learning framework enforces a switchback design, in which all nodes must adopt the same arm simultaneously. However, this approach does not account for scenarios in which the optimal arm may vary between nodes or subgroups. To address these limitations, \citet{zhang2024online} proposed a general learning framework, \texttt{MAB-N}, which simplifies the stochastic MABNI problem while allowing flexible adjustment of the action space through exposure mapping and clustering techniques. \texttt{MAB-N} generalizes the settings considered in \citet{jia2024multi} and \citet{agarwal2024multi}, treating them as special cases (see discussion in Section \ref{sec_mab-N}). Furthermore, \citet{xu2024linear} extended MABNI to the linear contextual bandit setting, incorporating a structured linear relationship between potential outcomes and interference intensity.

\paragraph{Trade-off between inference and regret.} A substantial body of research has focused on developing statistical methods for inference in stochastic MAB, often deriving statistical tests or central limit theorems while keeping the bandit algorithm largely unchanged~\citep{luedtke2016statistical, dimakopoulou2017estimation, dimakopoulou2019balanced, zhang2020inference, dimakopoulou2021online, hadad2021confidence, zhang2021statistical, han2022online, deshpande2023online, simchi2023multi}. These methods enable aggressive regret minimization but are subject to several key limitations: (i) they rely on the SUTVA, (ii) they assume that bandit rewards are i.i.d. samples from specific distribution families, and (iii) they do not support anytime-valid continual inference. Regarding the last limitation, to our knowledge, the only works that attempt continual inference in the adversarial bandit setting are~\citet{ham2023design, woong2023design, liang2023experimental}. However, these studies also assume SUTVA and lack a rigorous theoretical analysis of the inference-regret trade-off. To explore this inference-regret trade-off, researchers have first shifted their attention to a simpler problem: balancing estimation accuracy and regret minimization. To our knowledge, the first rigorous trade-off results were provided by \citet{simchi2023multi} in the stochastic MAB setting, though their approach remains constrained by the SUTVA assumption. \citet{duan2024regret} further argues that Pareto optimality—simultaneously achieving optimal regret and estimation accuracy—can be improved by introducing a covariate diversity assumption, provided that there is no network interference. More recently, \citet{zhang2024online} extended the trade-off results from \citet{simchi2023multi} to the stochastic \texttt{MAB-N} framework, accommodating network interference.

Additional discussions of the related work are provided in the Appendix. The relationship between our work and the most closely related studies is summarized in Table~\ref{tab:five_columns}.

\section{Preliminaries}\label{framework}

In this section, we first introduce the basic adversarial MABNI framework. Then, we present the techniques of exposure mapping and clustering and outline the adversarial \texttt{MAB-N} framework.

\subsection{Basic framework: adversarial MABNI}\label{MABNI}

We extend the classic single-unit adversarial bandit framework \citep{auer2002nonstochastic} to incorporate network interference \citep{zhang2024online}. Consider a network with $N$ units, represented by the set 
$\mathcal{U} = \{1,\dots,N\}$ and the adjacency matrix $ \mathbb{H} := \{h_{i,j}\}_{i,j\in[N]} \in \{0,1\}^{N \times N}$ (where $h_{i,j} = 1$ indicates unit $i$ and $j$ are neighbors, while $h_{i,j} = 0$ indicates otherwise). It is worth noting that full knowledge of $\bH$ is not strictly required; its necessity depends on the specific design introduced in the following section (see the discussion in Section \ref{sec_mab-N}). We assume that each unit has a $K$-armed set (action set) denoted as $\mathcal{K} = \{0,1, 2, \ldots, K-1\}$. At each round $t$, the learner must assign an arm to each unit, resulting in a super arm (a collection of arms across all units) represented as $A_t = (a_{1,t}, a_{2,t}, \ldots, a_{N,t}) \in \mathcal{K}^\mathcal{U}$. Suppose the super arm $A_t$ is pulled in round $t$, the reward derived by unit $i \in \mathcal{U}$ is $Y_{i,t}(A_t) \in [0,1]$, where $Y_{i,t}(\cdot): \mathcal{K}^\U \rightarrow \mathbb{R}$ represents the reward function of unit $i$ in round $t$. The terminal time $T$ is not pre-specified and cannot be known to the learner in advance. We define the set of all legitimate design-based bandit instances as $\bbE_0$, where a legitimate instance $\nu:= \{Y_{i,t}(A)\}_{A\in\mathcal{K}^\U,i\in\U,t\in[T]}$ satisfies $Y_{i,t}(A) \in [0,1]$ for all $A\in\mathcal{K}^\U$, $t\in[T]$ and $i\in\U$.

We aim to design a policy $\pi := (\pi_1, \dots, \pi_T)$. The $\pi_t$ is a rule that determines the super arm pulled in round $t$ based on the history $\bbH_{t-1}$:= $\{A_1, \{ Y_{i,1}(A_1) \}_{i\in\U}, \dots, A_{t-1}, \{ Y_{i,t-1}(A_{t-1}) \}_{i\in\U}\}$. Specifically, $\pi_t(A) = \bP(A_t = A \mid \bbH_{t-1})$. The performance of the policy is commonly measured by the cumulative regret \citep{auer2002nonstochastic,lattimore2020bandit}, defined as 
\begin{align}
\nonumber
  \mathcal{R}(T, \pi):=\max_{A \in \mathcal{K}^\mathcal{U}} \sum_{t=1}^T   \frac{1}{N} \sum_{i\in\U} Y_{i,t}(A) 
    -
    \mathbb{E}_\pi\biggl[
      \sum_{t=1}^T \frac{1}{N} \sum_{i\in\U} Y_{i,t}(A_t) 
    \biggr].       
\end{align}
The above-mentioned problem is far more challenging than the simple MAB problem (which only involves $K$ arms), as it involves $K^N$ possible super-arms, increasing the action space exponentially. As shown by \citet{zhang2024online} (see their Proposition 1), in certain difficult situations, any valid policy $\pi$ will incur regret that grows linearly with the time horizon $T$, i.e., $\mathcal{R}(T, \pi) = \Omega(T)$. To manage this complexity, we adopt the method of \citet{zhang2024online}, employing two key techniques: exposure mapping and clustering \citep{leung2022causal, zhang2024online} to reduce the effective dimensionality of the action space. These techniques enable the formulation of a unified framework, \texttt{MAB-N}, which captures a broad spectrum of learning settings.

\subsection{\texttt{MAB-N}}\label{sec_mab-N}

\paragraph{Exposure mapping \citep{leung2022causal}}
Exposure mapping is a common tool in causal inference for network interference that reduces the complexity of treatment assignments in networked settings. The core idea of exposure mapping is to compress these high-dimensional features of neighbors and network structure into a smaller set of exposure categories. Instead of labeling each individual as merely “treated” or “untreated,” we assign them an exposure level that reflects the degree or type of influence they experience, such as “having two treated neighbors” or “having at least one treated neighbor.” The definition of the exposure mapping follows \citep{leung2022causal,zhang2024online}:
\begin{equation}
s \equiv \bbS(i, A, \mathbb{H}), \text{~where~}
\bbS: \mathcal{U} \times \mathcal{K}^{\U} \times \{0,1\}^{N \times N} \rightarrow \U_s, \label{exposure_def}
\end{equation}
where $s$ denotes the exposure arm, $\bbS$ the exposure mapping, 
and $\U_{s}$ its output space, referred to as the exposure-arm set, 
with cardinality $|\mathcal{U}_{s}| = d_{s}$. Intuitively, the exposure mapping reduces the original super arm space of size $K^N$ to an exposure arm space of size $d_s$. We define $S = \{\textbf{S}(i,A,\mathbb{H})\}_{i \in \U} \equiv (s_{1},...,s_{N})$ as the \textit{exposure super arm}. This allows us to decompose the policy $\pi_t(\cdot)$ and define the expected exposure mapping-based reward:
\begin{equation}\label{definitionpolicy}
\begin{aligned}
 &{\pi_t(A)} \equiv  {\mathbb{P}_{}(A_t = A \mid S_t) } {{\mathbb{P}_{}}({S_t} \mid \bbH_{t-1})},\quad \Tilde{Y}_{i,t}({S}_{t}) := \sum_{A \in \mathcal{K}^{\U}} Y_{i,t}(A) \mathbb{P}_{}(A_t =A \mid {S}_t).
\end{aligned}
\end{equation}
Here $S_{t}$ denotes the exposure super arm selected by the algorithm in round~$t$ 
based on the history~$\mathcal{H}_{t-1}$, 
and the policy $\pi_{t}(A)$ is represented by a two-stage sampling procedure: 
it first draws $S_{t}\sim\mathbb{P}(\cdot\mid\mathcal{H}_{t-1})$ 
and then samples $A_{t}\sim\mathbb{P}(\cdot\mid S_{t})$. The second line of the above equation is a generalized notation of~\citet{leung2022causal}. Notably, $\mathbb{P}(A_t = A \mid S)$ represents a fixed sampling rule that can be manually defined by the learner before the learning starts. Typically, the probability of selecting \( A_t = A \) given \( S \) is zero if \( S \) does not match the set \( \{\textbf{S}(i,A,\mathbb{H})\}_{i \in \U} \). Conversely, if \( S \) is equal to this set, then the probability of choosing \( A \) is strictly positive, i.e., \( \mathbb{P}(A_t = A \mid S) > 0 \). In this context, the expected reward of $S$ (i.e., $\Tilde{Y}_{i,t}(S)$) depends solely on the definition of the exposure mapping $\bbS$ and the network topology $\mathbb{H}$. 

\paragraph{Clustering.}
We define the clustering set as $\mathcal{C} := \{\mathcal{C}_q\}_{q \in [C]}$, where $C = |\mathcal{C}|$ represents the total number of clusters. The clusters are assumed to be disjoint, meaning that for any $i \neq j$, $\mathcal{C}_i \cap \mathcal{C}_j = \varnothing$, and collectively exhaustive, such that $\bigcup_{q \in [C]} \mathcal{C}_q = [N]$. For any $i \in [N]$, we denote $\mathcal{C}^{-1}(i)$ as the cluster containing $i$. Such an operation is common and necessary, otherwise, the total arm space is exponentially large.

\paragraph{Framework of \texttt{MAB-N}.} We define the legitimate exposure super arm set as $\U_\bbE := \U_\mathcal{C} \cap \U_\mathcal{O}$, where $\U_{\mathcal{O}} := \big\{\{\textbf{S}(i, A, \mathbb{H})\}_{i \in \U}: A \in \mathcal{K}^{\mathcal{U}} \big\}$ ensuring that $S \in \U_\mathcal{O}$ is compatible with the original arm set $\mathcal{K}^\U$, and $\U_{\mathcal{C}} := \big\{S_t: \forall i,j \in \U,\ \text{if}\ \mathcal{C}^{-1}(i) = \mathcal{C}^{-1}(j),\ \text{then}\ s_{i,t} = s_{j,t} \big\}$ denoting all kinds of cluster-wise switchback exposure super arms. For instance, if $\mathcal{U}_s \in \{0,1\}, N=4, \mathcal{C}_1 = \{1,2\}, \mathcal{C}_2 = \{3,4\}$, then $\mathcal{U}_{\mathcal{C}} = \{(k_1,k_1, k_2,k_2):{k_1,k_2 \in \{0,1\}}\}$. Hence, the cardinality of the exposure super arm space satisfies $\vert \U_\mathcal{E} \vert \le |d_s|^C$. The word ``legitimate" means in each round, the policy can only select an exposure super arm $S_t$ in $\U_\bbE$ and sample the $A_t$ according to $\bP(A_t = A \mid S_t)$.
The exposure mapping (which controls $d_s$) and clustering (which controls $C$) allow us to manage the action space; they only need to satisfy the following condition:

\begin{condition}\label{armspace}
  The exposure mapping $\bbS$ and $\mathcal{C}$ should ensure that $2 \leq \vert \U_\bbE \vert \leq T$.
\end{condition}

In addition, we define $\bbY_t(S) = \frac{1}{N} \sum_{i\in\U} \tilde{Y}_{i,t}(S)$ as the expected average reward of the exposure super arm $S \in \U_\bbE$ in round $t$. The reward in round $t$ \( R_t(S_t) \) follows $R_t(S_t) = \frac{1}{N} \sum_{i \in \mathcal{U}} Y_{i,t}(A_t),$ $\text{where } A_t \sim \mathbb{P}(A_t = A \mid S_t)$.

\paragraph{\texttt{MAB-N} is a unified framework.} It is important to note that \texttt{MAB-N} is not parallel to the learning settings in \citet{jia2024multi, agarwal2024multi}; rather, it provides a more general framework that encompasses these settings. In the following, we provide several illustrative examples: \textbf{Example (i).} \textit{Classic MAB}~\citep{Auer2002FinitetimeAO, simchi2023multi} corresponds to the case where $N = 1$, that is, a single unit without network effects, and the exposure mapping is defined as $\boldsymbol{S}(1, A, \mathbb{H}):= A$, where $A \in \mathcal{K}$. \textbf{Example (ii).} \citet{agarwal2024multi} adopt an exposure mapping of the form $\boldsymbol{S}(i, A, \mathbb{H}) := A\bm{e}_i$ and set $C = N$, meaning each unit is assigned to its own cluster. \textbf{Example (iii).} \citet{jia2024multi} also define $\boldsymbol{S}(i, A, \mathbb{H}) := A\bm{e}_i$, but use $C = 1$, assigning all units to a single cluster. This models the global proportion of treatment at each round $t$. \textbf{Example (iv).} The exposure mapping and clustering framework can also be traced back to the offline causal inference literature. Suppose $\sum_j h_{ij} > 0$ for all $j \in \mathcal{U}$. The exposure mapping can be defined as $\boldsymbol{S}(i, A, \mathbb{H}) := \mathbf{1}\left\{\frac{\sum_{j\in\mathcal{U}} h_{ij} a_j}{\sum_{j\in\mathcal{U}} h_{ij}} \in \left[0, \tfrac{1}{2}\right)\right\}$,
which is adapted from the offline setting~\citep{leung2022causal, gao2023causal}. 

As shown in the above examples, \texttt{MAB-N} is a unified framework that captures a wide range of learning settings. Studying it effectively subsumes many existing scenarios. For example, to model a switchback design as in \citet{jia2024multi}, one can adopt the exposure mapping and clustering described in Example (iii). Moreover, \texttt{MAB-N} also enables the exploration of novel frameworks that have not been considered in prior online settings, such as the one presented in Example (iv).

\paragraph{Is $\bH$ necessarily known?} We emphasize that whether the adjacency matrix $\bH$ must be known a priori depends entirely on how the exposure mapping $\bbS(\cdot)$ is defined. For example, if our setting reduces to the scenario in \citet{leung2022causal}---namely, when the exposure mapping depends on all first-order neighbours---then the neighbourhood information in $\bH$ must indeed be known in advance. In contrast, if our exposure mapping simply uses each node's cluster index, then $\bH$ can remain unknown. 
Overall, we include $\bH$ in our setup in order to focus on a unified framework, and this does not imply that all information in $\bH$ always needs to be learned. 

Finally, we introduce the definition of the ATE~\citep{leung2022causal,liang2023experimental}.

\begin{definition}[ATE]\label{definitionATE}
    The ATE between exposure super arm $S_i,S_j \in \U_\bbE$ in round $t$ is $\bar{\tau}_t(S_i,S_j) = \frac{1}{t}\sum_{t^\prime = 1}^t \tau_t(S_i,S_j) = \frac{1}{t}\sum_{t^\prime = 1}^t \big( \bbY_{t^\prime}(S_i) - \bbY_{t^\prime}(S_j) \big)$.
\end{definition}

\subsection{Learning objectives}

\paragraph{Objective 1: Regret minimization.} Based on the setting of the \texttt{MAB-N}, we can refine the definition of regret in Section \ref{MABNI}:
\begin{align}
\nonumber
\mathcal{R}(T, \pi)
=
\max_{S \in \mathcal{U}_{\mathcal{E}}} \sum_{t=1}^T  \bbY_t(S) 
- \mathbb{E}_{\pi} \left[\sum_{t=1}^T R_t(S_t)\right].
\end{align}

\paragraph{Objective 2: Continual inference.} For all \( S_i, S_j \in \mathcal{U}_\mathcal{E} \), our objective is to design a (\(1 - \tilde{\delta}\)) CS \( \{ I_t(S_i, S_j) \}_{t=1}^\infty \), where each \( I_t(S_i, S_j) \) is an interval and $\tilde{\delta}$ is the probability parameter such that $\mathbb{P} \left( \forall t \geq 1, \, \bar{\tau}_t(S_i,S_j) \in I_t(S_i, S_j) \right) \geq 1 - \tilde{\delta}$.

\paragraph{Objective 3: ATE estimation error minimization.} We aim to design estimators $\hat{\Delta}_T(S_i, S_j)$ for all $S_i, S_j \in \U_\bbE$, to minimize the maximum ATE estimation error \citep{simchi2023multi,zhang2024online} defined as $e_\nu(T, \hat{\Delta}) = \max_{S_i, S_j \in \U_\bbE} \bE\big[|\hat{\Delta}_T(S_i, S_j) - \bar{\tau}_T(S_i, S_j)|\big]$.

Recalling Figure~\ref{fig:enter-label}, simultaneously addressing Obj.~1-3 is a shared concern among online learning and statistical researchers, with the former primarily focusing on Obj.~1 and the latter on Obj.~2-3. However, achieving both Obj.~1 and Obj.~2-3 simultaneously is often challenging; essentially, there is a trade-off between the two. When we excessively prioritize the estimator's accuracy (e.g., through independent random sampling), we may fail to adequately explore the optimal strategy, resulting in regret that it does not remain sublinear in time $T$. Conversely, if we focus solely on identifying the optimal strategy, we naturally overlook the measurement of the reward gap, which can lead to uncontrolled variance in the estimator. In the following section, we provide a rigorous description of the relationship between Objective 1 and Objective 3.

\section{Pareto optimality results}\label{4}

 In this paragraph, we aim to construct the theoretical optimal trade-off between the ATE estimation accuracy and the regret.

\begin{theorem}\label{trade-off}
   Given any online decision-making policy $\pi$, and any $\mathcal{S}$ and $\mathcal{C}$ that satisfy Condition~\ref{armspace}, the trade-off between regret and ATE estimation exhibits
    \begin{equation}
    \begin{aligned}
    \inf _{\hat{\Delta}} \max _{\nu \in \mathcal{E}_0} \Big( \sqrt{\mathcal{R}_{\nu}(T, \pi)}  e_\nu(T, \hat{\Delta}) \Big) =  \Omega_{K,T}(\sqrt{|\mathcal{U}_{\mathcal{E}}|}),
    \end{aligned}
    \end{equation}
where $\mathcal{R}_{\nu}$ and $e_{\nu}$ denote, respectively, the regret and the maximum ATE estimation error under instance~$\nu$.
\end{theorem}

\paragraph{The sketch of proof.} 

To prove the lower bound, we construct two adversarial bandit instances that differ only in the expected reward associated with one exposure super arm $S$, while all other distributions remain identical. This creates a small but fixed difference in the average treatment effect (ATE) between $S$ and another arm $S^{\prime}$, yet renders the two instances statistically hard to distinguish. The crux of the argument is that unless the learner samples $S$ sufficiently often, it cannot accumulate enough information to detect this perturbation. Consequently, any estimator of the ATE between $S$ and $S^{\prime}$ will exhibit a large error due to insufficient exploration. Formally, the argument applies tools from information theory, specifically Le Cam's two-point method and a KL-divergence bound, to show that accurate estimation of the ATE requires distinguishing between the two constructed environments, which in turn necessitates a minimum number of pulls of arm $S$. This induces a direct tension between the estimation error $e(T, \hat{\Delta})$ and the cumulative regret $\R(T, \pi)$: minimizing one forces the other to grow. By optimizing this trade-off, we derive the lower bound as above, which highlights a fundamental information-theoretic limit in adversarial bandits under network interference. The novelty lies in extending classical bandit lower-bound techniques to the setting of networked exposure mappings and adversarial reward generation, preserving the sharp dependency on the effective arm space size $\left|\mathcal{U}_\bbE\right|$.

Theorem~\ref{trade-off} establishes the fundamental trade-off between the estimation, namely, the statistical power, and the cumulative regret, namely, the learning efficiency. For instance, when the estimation achieves $T^{-1/2}$ estimation, we claim that, unfortunately, the regret will exhibit as $\Omega(T)$. In contrast, when we omit the estimation of ATE and solely figure out the best arm, the regret will converge. This guideline essentially encourages practitioners to carefully and reasonably design estimators and evaluate their convergence performance concerning $T$. When practitioners are more inclined to estimate the reward gap between different arms rather than pursuing the optimal policy --- such as in scenarios where hospitals, during a specific period of a pandemic, aim to assess the efficacy of treatments more accurately --- efforts should be directed toward actively designing estimators with higher convergence efficiency. Practitioners should also be prepared to accept the trade-off of potential losses in regret convergence resulting from this approach.

\section{Asymptotic Confidence Sequence and Main Algorithm}\label{5}

In this section, we first introduce a technique called asymptotic CS, which facilitates continual inference of the ATE as defined in Definition \ref{definitionATE}. Next, we propose our algorithm \texttt{EXP3-N-CS} that integrates asymptotic CS to achieve three objectives. 

\subsection{Asymptotic CS and MAD}
CS is a series of confidence intervals that remain uniformly valid over time \citep{darling1967confidence,waudby2021time}. Unlike traditional confidence intervals, which are limited to inference at a pre-specified terminal time \( T \), a CS enables continual inference throughout the process. This allows for adaptive decisions regarding experiment termination or continuation, as the learning algorithm does not need to know or define the time horizon \( T \) in advance. Instead, the algorithm can continuously utilize the CS for inference, concluding the experiment once satisfactory learning outcomes are achieved. We introduce the concept of asymptotic CS, first developed by \citet{waudby2021time}. 

\begin{definition}[Asymptotic ($1 - \tilde{\delta}$) CS]
Suppose there exists an (unknown) non-asymptotic ($1 - \tilde{\delta}$) CS $\{\hat{\mu}_t \pm C_t^*\}_{t = 1}^\infty$ for a sequence of target parameter $\{\mu_t\}_{t = 1}^\infty$ and a CS $\{\hat{\mu}_t \pm \hat{C}_t\}_{t = 1}^\infty$ such that $\frac{\hat{C}_t}{C_t^*}\xrightarrow{a.s.} 1,$
then $\{\hat{\mu}_t \pm \hat{C}_t\}_{t = 1}^\infty$ is an asymptotic ($1 - \tilde{\delta}$) CS for $\{\mu_t\}_{t = 1}^\infty$.
\end{definition}

Our Asymptotic CS for \texttt{MAB-N} is defined in the following proposition.

\begin{proposition}[Asymptotic CS for \texttt{MAB-N}]\label{definitionCS}
     We define the asymptotic CS for $S_i, S_j \in \U_\bbE$ as $\{\hat{\bar{\tau}}_t(S_i, S_j) \pm \hat{C}_t(S_i, S_j)\}_{t=1}^\infty$. The IPW estimator $\hat{\bar{\tau}}_t(S_i, S_j)$ is defined as $\frac{1}{t} \sum_{t^\prime = 1}^t \hat{\tau}_{t^\prime}(S_i, S_j) = \frac{1}{t} \sum_{t^\prime = 1}^t \left(\frac{\bone\{ S_{t^\prime} = S_i \} R_{t^\prime}(S_{t^\prime})}{\pi_{t^\prime}(S_i)} - \frac{\bone\{ S_{t^\prime} = S_j \} R_{t^\prime}(S_{t^\prime})}{\pi_{t^\prime}(S_j)}\right)$,
     which serves to estimate $\bar{\tau}_t(S_i, S_j)$. The CS width $\hat{C}_t(S_i, S_j) = \sqrt{\frac{2 (\hat{\V}_t(S_i, S_j) \eta^2 + 1)}{t^2 \eta^2} \log\left( \frac{\sqrt{\hat{\V}_t(S_i, S_j) \eta^2 + 1}}{\tilde{\delta}} \right)}$, where $\hat{\V}_t(S_i,S_j) = \sum_{t^\prime = 1}^t \Big( \frac{1}{\pi^{\text{MAD}}_{t^\prime}(S_i)} + \frac{1}{\pi^{\text{MAD}}_{t^\prime}(S_j)} \Big)$ and $\eta$ is an arbitrary positive parameter.
\end{proposition}

 Asymptotic CS can appear as a plug-in module that operates independently of any specific algorithm. Its performance is based on the following assumption: 

\begin{assumption}\label{assumption1}
We require that the cumulative conditional variances $\V_t(S_i, S_j) = \sum_{t^\prime = 1}^t \textbf{V}(\hat{\tau}_{t^\prime}(S_i, S_j) \mid \mathcal{F}_{t^\prime})$ grow at least linearly with $t$ for all $S_i, S_j \in \U_\bbE$, that is, $\V_t(S_i, S_j) = \Omega(t)$, where $\mathcal{F}_t$ denotes the sigma algebra that contains $\{\tilde{Y}_{i,t'}(S)\}_{S\in\U_\bbE,i\in\U,t'\in[t]}$ and $\bbH_t$.
\end{assumption}
The above assumption is weaker than the one made by \citet{simchi2023multi}, which assumes that the expected reward gap of each pair of arms is $\Theta(1)$ (stochastic setting). We should mention that our assumption is relatively stronger than the assumption in \citet{waudby2021time,ham2023design}, which only requires the cumulative conditional variance $\V_t(S_i, S_j) \rightarrow \infty$ when $t\rightarrow\infty$, but does not assume a linear growth rate. However, this assumption should hold in the most realistic experimental settings, provided that instances where there exists a time $t'$ such that $\exists S$, $\bbY_t(S) = 0$ for all $t > t'$ are rare in practice and may indicate practical problems with the experiment~\citep{liang2023experimental}. The asymptotic CS presented in Proposition \ref{definitionCS} can be incorporated into various classic adversarial bandit algorithms, such as EXP3. However, algorithms like EXP3 primarily focus on minimizing regret, which often leads to sampling the exposure super arm with low rewards at a low probability. This behavior can reduce the accuracy of our IPW estimator in estimating low-reward exposure super arms and significantly weaken the inference power of the asymptotic CS. Therefore, it is essential to ensure that the algorithm explores each exposure super arm with sufficiently high probability. To this end, we incorporate the MAD \citep{liang2023experimental}, a modular component that can be integrated into various algorithms to promote effective exploration.

\begin{definition}[MAD]
Let the probability of Algorithm \( \text{ALG} \) pulling arm \( S \) in round \( t \) be denoted by \( \pi_t^{\text{ALG}}(S) = \mathbb{P}_{\text{ALG}}(S_t = S \mid \bbH_{t-1}) \), where $\mathbb{P}_{\text{ALG}}$ denotes the probability taken with respect to ALG. After applying MAD, the probability of pulling the exposure super arm \( S \) in round \( t \) is given by $\pi_t^{\text{MAD}}(S) = \mathbb{P}_{\text{MAD}}(S_t = S\mid\bbH_{t-1}) =  \frac{1}{\vert \U_\bbE \vert}\delta_t + (1 - \delta_t)\pi_t^{\text{ALG}}(S)$,  where \( \delta_t \in [0,1] \) is a time-varying parameter and $\mathbb{P}_{\text{MAD}}$ denotes the probability taken concerning MAD. It is easy to verify that $\pi_t^{\text{MAD}}(S) \in [0,1]$ for all $S\in\U_\bbE$ and $\sum_{S\in\U_\bbE}  \pi_t^{\text{MAD}}(S) = 1$.
\end{definition}

 MAD can balance the trade-off between regret minimization and additional exploration. Consider two special cases: $\delta_t = 0$ and $\delta_t = 1$. When $\delta_t = 0$, the policy becomes $\pi_t^{\text{MAD}} = \pi_t^{\text{ALG}}$, entirely focusing on minimizing regret (as we suppose ALG intends to minimize the regret). On the other hand, when $\delta_t = 1$, the policy becomes $\pi_t^{\text{MAD}}(S) = \frac{1}{\vert \U_\bbE \vert}$ for all $S\in\U_\bbE$ (uniformly samples $S_t$ from $\U_\bbE$), entirely prioritizing exploration. The following Theorem \ref{CS} shows that with a specific setup, the CS proposed in Proposition \ref{definitionCS} is a valid asymptotic ($ 1-\tilde {\delta}$) CS.

\begin{theorem}[Performance of the Asymptotic CS]\label{CS}
   Suppose $\bbS$ and $\mathcal{C}$ satisfy Condition \ref{armspace} and Assumption \ref{assumption1} holds. For all \( S_i, S_j \in \U_\bbE \), consider the sequence of random variables \( \big( \hat{\tau}_t(S_i, S_j) \big)_{t=1}^\infty \), where the probability of observing \( S_t = S \) at time \( t \) is given by $\pi_t^{\text{MAD}}(S) = \frac{1}{\vert \U_\bbE \vert} \delta_t + (1 - \delta_t) \pi_t^{\text{ALG}}(S)$. We set \( \delta_t = \frac{1}{t^{\alpha}} \) which satisfies $\alpha \in [0,\frac{1}{2})$. The CS in Proposition \ref{definitionCS} forms a valid asymptotic CS for \( (\bar{\tau}_t(S_i, S_j))_{t = 1}^\infty \) with confidence level \( 1 - \tilde{\delta} \) and the CS width \( \hat{C}_t(S_i,S_j) = \tilde{O}(\vert \U_\bbE \vert^{\frac{1}{2}} t^{\frac{\alpha-1}{2}}) \).
\end{theorem}

\subsection{Main algorithm}\label{5.2}

In the previous section, we demonstrated that both the asymptotic CS and the MAD can be integrated into learning algorithms such as EXP3. In this section, we analyze the performance of the resulting algorithm, which we refer to as \texttt{EXP3-N-CS}. 

\begin{algorithm}[t]
    \caption{\texttt{EXP3-N-CS}}
    \begin{algorithmic}[1]\label{alg1}
        \STATE \textbf{Input:} arm set $\A$, unit set $\U$, exposure super arm set $\U_\mathcal{E}$, sequence $\{\mathcal{L}_m\}_{m=1}^\infty$
        \FOR{\(t = 1, 2, \dots\)}
            \STATE Compute $m$ such that $t \in \mathcal{L}_m$, set $\epsilon_m = \sqrt{\frac{\log(\vert \U_\bbE \vert)}{\vert \U_\bbE \vert 2^{m-1}}}$
            \IF{$t = t_m$}
            \STATE For all $S \in \U_\bbE$, set
            $\pi_t^{\text{ALG}}(S) = \frac{1}{\vert \U_\bbE \vert}$
            \ELSE
            \STATE For all $S \in \U_\bbE$, set
            $\pi_t^{\text{ALG}}(S) = \frac{\exp\left( \epsilon_m \hat{R}_{\mathcal{L}_m,t-1}(S) \right)}{\sum_{S' \in \U_\bbE} \exp\left( \epsilon_m \hat{R}_{\mathcal{L}_m,t-1}(S^\prime) \right)}$
            \ENDIF
            \STATE For all $S \in \U_\bbE$, set $\pi_t^{\text{MAD}}(S) = \frac{1}{\vert \U_\bbE \vert} \delta_t + (1 - \delta_t) \pi_t^{\text{ALG}}(S)$ 
            \STATE Sample $S_t$ based on $\pi_t^{\text{MAD}}$, implement \(\texttt{Sampling}(S_t)\) and observe the rewards $R_t(S_t)$
            \STATE For all $S_i,S_j\in\U_\bbE$, construct the confidence sequence $\hat{\bar{\tau}}_t(S_i, S_j) \pm \hat{C}_t(S_i, S_j)$
            \STATE For all $S\in\U_\bbE$, set $\hat{R}_{\mathcal{L}_m,t}(S) = \sum_{t^\prime=t_m}^{t} 1 - \frac{\bone\{S_t = S\} \left( 1 - R_t(S_t) \right)}{\pi^{\text{MAD}}_t(S)}$
        \ENDFOR
        \STATE Return $\hat{\Delta}^{(i,j)}_T = \hat{\bar{\tau}}_T(S_i, S_j)$ for all $S_i,S_j\in\U_\bbE$
    \end{algorithmic}
\end{algorithm}

\begin{algorithm}[t]
\caption{\texttt{Sampling}}
	\begin{algorithmic}[1]\label{alg2}
        \STATE \textbf{Input:} $S_t$
        \STATE Derive the set of real super arm $\{Z_{l'}\}_{l'\in[l]}$ such that for all $Z_{l'}$, $\{\bbS(i,Z_{l'},\bbH)\}_{i\in\U} = S_t$
        \STATE Sample $A_t$ from set $\{Z_{l'}\}_{l'\in[l]}$ based on $\bP(A_t = Z_{l'} \mid S_t)$, pull $A_t$,
        and observe reward $R_t(S_t) = \frac{1}{N}\sum_{i\in\U}Y_{i,t}(A_t)$
	\end{algorithmic}  
\end{algorithm}

\texttt{EXP3-N-CS} is designed to achieve three learning objectives. We ensure that the algorithm does not rely on prior knowledge of $T$ by employing the doubling trick \citep{besson2018doubling}. We define the time interval $\mathcal{L}_m := \{t_m,\dots, t_m + 2^{m-1} - 1\}$, where $t_1 = 1$ and $t_m = 1 + \sum_{m'=0}^{m-2} 2^{m'}$ for all $m>1$. The algorithm begins by computing the policy \( \pi_t^{\text{ALG}}(S) \) (ALG equals to EXP3) for all \( S \in \U_\bbE \) using the standard EXP3 technique (line 4-8). Next, it adjusts the policy to derive \( \pi_t^{\text{MAD}}(S) \) for all \( S \in \U_\bbE \) based on the MAD (line 9). The algorithm then samples \( S_t \) based on \( \pi_t^{\text{MAD}} \) and subsequently samples \( A_t \) conditioned on \( S_t \) (line 10 and Algorithm \ref{alg2}). Algorithm \ref{alg2} will first derive a real super arm candidate set $\{ Z_{l^\prime} \}_{l' \in [l]}$ such that $\{\bbS(i,Z_{l'},\bbH)\}_{i\in\U} = S$, $\forall l' \in [l]$. Then, it will sample $A_t$ from $\{ Z_{l^\prime} \}_{l' \in [l]}$ based on $\mathbb{P}(A_t = A \mid S_t)$ (note that for all $A\not\in\{ Z_{l^\prime} \}_{l' \in [l]}$, $\mathbb{P}(A_t = A \mid S_t) = 0$). Using the asymptotic CS proposed in Proposition \ref{definitionCS}, the algorithm constructs the CS to estimate the ATE in each round (line 11). Note that the $\pi_t(S)$ in the asymptotic CS should be replaced with $\pi^{\text{MAD}}_t(S)$. Finally, after the algorithm terminates the iteration, the algorithm outputs \( \hat{\Delta}_T(S_i, S_j) = \hat{\bar{\tau}}_T(S_i, S_j) \) as the estimated ATE (line 14).

\begin{theorem}\label{ATEbound}
   Following the setting in Theorem \ref{CS}:
   
   (i) \textbf{(Estimation error upper bound)} For all $S_i, S_j \in \U_\bbE$, define $\hat{\Delta}_T(S_i, S_j) := \hat{\bar{\tau}}_T(S_i, S_j)$. Then, we have $\mathbb{E}[\vert \hat{\Delta}_T(S_i, S_j) - \bar{\tau}_T(S_i, S_j) \vert] = \tilde{O}\big(\vert \U_\bbE \vert^{\frac{1}{2}} t^{\frac{\alpha - 1}{2}} \big)$.

    (ii) \textbf{(Regret upper bound)}  The regret of the \texttt{EXP3-N-CS} can be upper bounded by $\R(T,\pi^{\text{MAD}}) = \tilde{O}\big(\sqrt{\vert \U_\bbE \vert T} + T^{1 - \alpha}\big)$.

    (iii) \textbf{(Pareto-optimality)}  For all legitimate instances \( \nu \in \bbE_0 \), select $\alpha$ such that $\sqrt{\vert \U_\bbE \vert T} \le T^{1-\alpha}$ and $\alpha\in[0,\frac{1}{2})$, then \texttt{EXP3-N-CS} guarantees $e_\nu(T, \hat{\Delta})\sqrt{\R_\nu(T, \pi^{\text{MAD}})} = \tilde{O}\big(\sqrt{|\U_\bbE|}\big)$.
\end{theorem}

From Theorem~\ref{ATEbound}~(iii), we conclude that \texttt{EXP3-N-CS} achieves the Pareto-optimal trade-off established in Theorem~\ref{trade-off}. There is no need to choose $\alpha$ larger than $\tfrac{1}{2}$, as doing so does not reduce the regret but instead deteriorates the estimation accuracy. Furthermore, although our analysis centers on \texttt{EXP3-N-CS}, the underlying design principles naturally extend to a broader class of bandit algorithms, owing to the strong modularity and composability of the Asymptotic~CS and MAD components. In particular, for any base algorithm that achieves a regret bound of the form \(\tilde{O}(\sqrt{|\mathcal{U}_{\mathcal{E}}| T})\) under the \texttt{MAB-N} framework, the performance guarantee in Theorem~\ref{ATEbound} can be easily extended. This is because the MAD adopts the form 
$\pi_{t}^{\mathrm{MAD}}(S) 
= \frac{1}{|\mathcal{U}_\bbE|}\delta_t 
  + (1-\delta_t)\pi_{t}^{\mathrm{ALG}}(S)$,
where the first term (with coefficient $\delta_t$) is introduced to improve ATE estimation, and the second term $\pi_{t}^{\mathrm{ALG}}(S)$ is the base algorithm aimed at regret minimization. To analyze the regret, we decompose the total regret into two components corresponding to the two terms in the strategy, and then separately upper bound each. 
As shown in Theorem~5.6~(ii), the first part of the regret scales as $\tilde{\mathcal{O}}\big(\sqrt{|\mathcal{U}_\bbE|T}\big)$, which comes from the $(1-\delta_t)\pi_{t}^{\mathrm{ALG}}(S)$ term, and the second part scales as $\tilde{\mathcal{O}}(T^{1-\alpha})$, arising from the $\frac{1}{|\mathcal{U}_\bbE|}\delta_t$ term, which is independent of the specific base algorithm. Therefore, we can easily analyze the overall regret upper bound as 
$\tilde{\mathcal{O}}\big(\sqrt{\vert \U_\bbE\vert T} + T^{1-\alpha}\big)$. In this sense, the Asymptotic~CS and MAD components serve as general-purpose mechanisms that facilitate balancing Objectives~1--3 across a broad range of algorithms.

We now present guidance on selecting $\alpha$ by combining theoretical insights with practical considerations. Specifically, setting $\alpha = 0$ results in linear regret, but ensures a fast convergence rate for the CS width and the ATE, specifically $\tilde{O}\big(\vert \U_\bbE \vert^{\frac{1}{2}} t^{-\frac{1}{2}} \big)$. In addition, setting $\alpha$ such that $\sqrt{\vert \U_\bbE \vert T} = T^{1-\alpha}$, we can minimize the regret to the level of $\tilde{O}\big(\sqrt{\vert \mathcal{U}_\bbE \vert T}\big)$ while achieving statistical inference with $\tilde{O}\big(\vert \mathcal{U}_\bbE \vert^{\frac{1}{4}} T^{-\frac{1}{4}}\big)$. To demonstrate its practical selection, consider the following example. When treating a group of critically ill patients, the primary goal is to minimize regret by assigning treatments that are currently believed to be most effective, thereby improving their immediate survival chances within the network. 
Conversely, for patients with milder symptoms and stable conditions, it can be advantageous to explore less-certain yet promising treatments, as doing so improves the evaluation of the relative effectiveness of different options. 
This enhanced inference leads to more precise ATE estimation and, in turn, supports better-informed treatment decisions for future populations. 
Accordingly, one may prefer selecting $\alpha$ closer to $\tfrac{1}{2}$ in the former scenario, prioritizing regret minimization, and closer to $0$ in the latter, emphasizing pure exploration.

\section{Experiments}\label{experiments}

\begin{figure*}[t]
\centering
 \subfigbottomskip=4pt
	\subfigcapskip=-5pt 
	\subfigure[Cumulative regret]{
\includegraphics[width=0.31\linewidth]{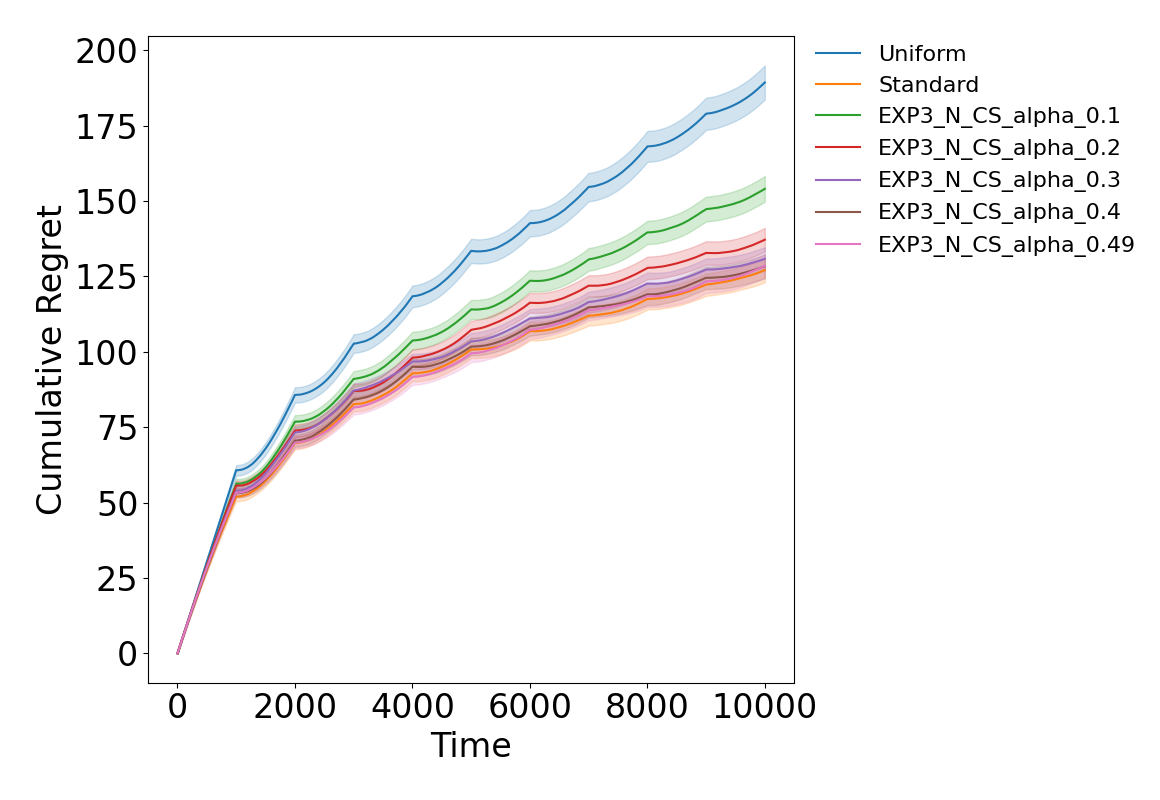}\label{fig2b}}
         \subfigure[CS width]{
\includegraphics[width=0.31\linewidth]{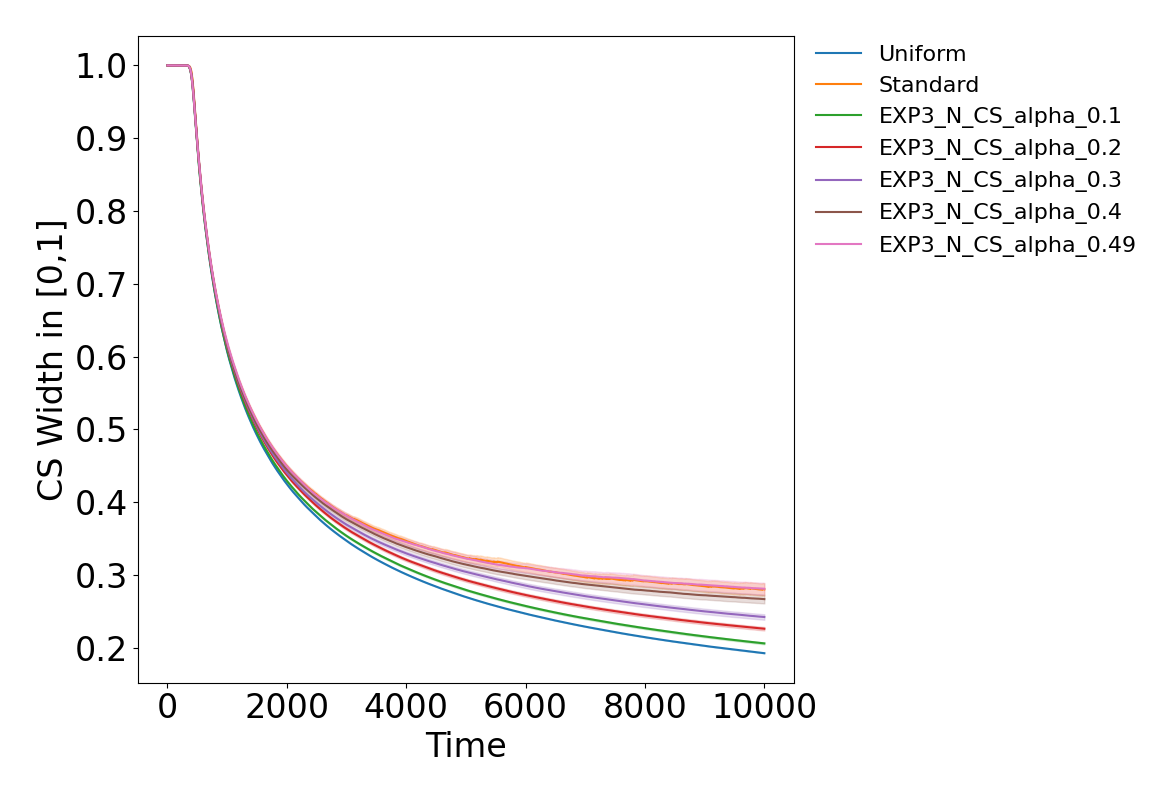}\label{fig2c}}
\subfigure[Maximum ATE estimation error]{
\includegraphics[width=0.34\linewidth]{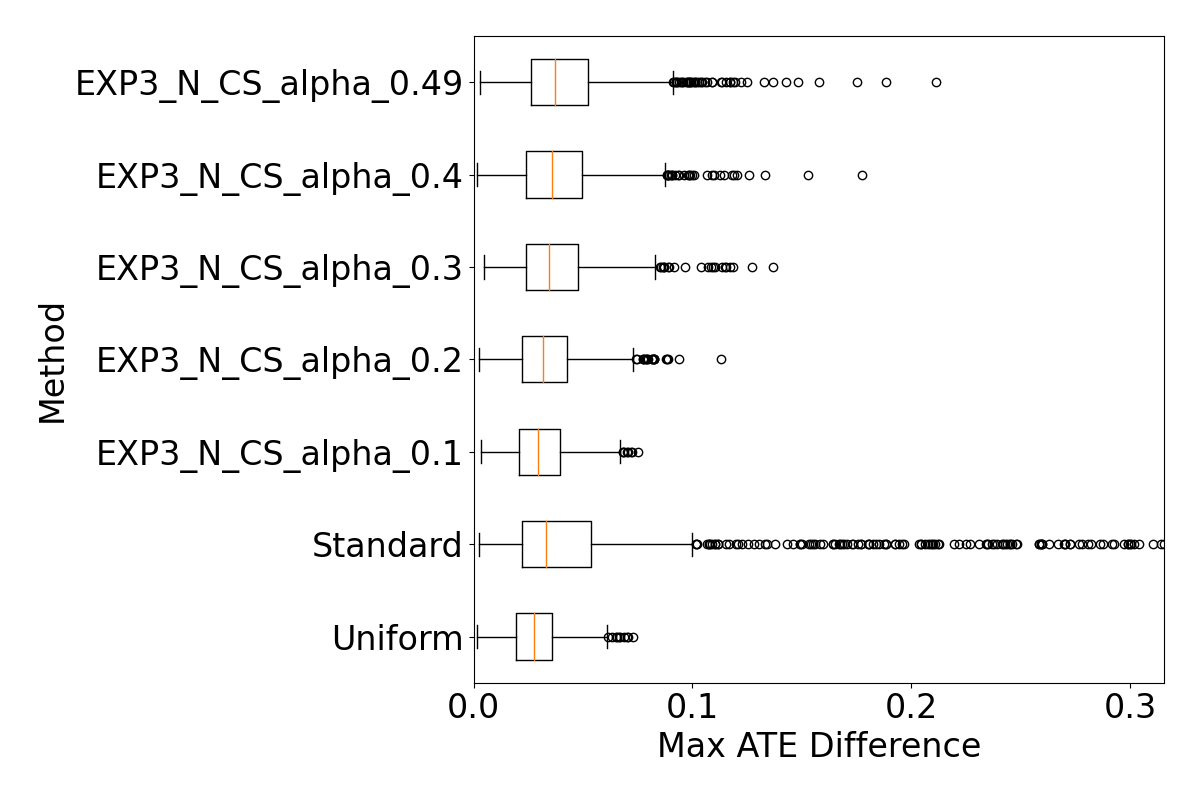}\label{fig2d}}
\caption{Experimental results.}
\end{figure*}

In this section, we demonstrate the empirical performance of our \texttt{EXP3-N-CS} by some simulation studies. The code is available at: \url{https://github.com/TheoryMagic/Design-based-Bandits}.

\paragraph{Setup.} We consider a network consisting of 101 units. Specifically, there is one center cluster $C_1 = \{1\}$ that contains a single unit, which is connected to every unit in the five outer clusters. Each of the outer clusters contains 20 units. We set the action set $\mathcal{K} = \{0,1\}$. Additionally, we define the exposure mapping inspired by \citep{leung2022causal,gao2023causal}, expressed as $ \bbS(i, A, \bH) = \bone\left\{\frac{\sum_j h_{i,j} \times a_{j}}{\sum_j h_{i,j}} \in \left[0, \frac{1}{2}\right)\right\}$, exploring the influence of the proportion of action $1$ taken among all the neighbors of each unit. The exposure mapping implies $d_s = 2$. For all \( S \in \mathcal{U}_{\bbE} \), we define $\bP(A_t = A \mid S)$ as uniform sampling, and \(\mathcal{Y}_t(S) = \frac{1}{N} \sum_{i \in \mathcal{U}} Y_{i,t}(A)\) for all \( A \) such that \(\{ \mathbf{S}(i, A, \mathbb{H}) \}_{i \in \mathcal{U}} = S\). Besides, we let $\bbY_t(S)$ be sampled from a Bernoulli distribution. The mean of this Bernoulli distribution is uniformly resampled from $[0,1]$ every 1000 rounds. We set the trade-off parameter of \texttt{EXP3-N-CS} to 
$\alpha \in \{0.1, 0.2, 0.3, 0.4, 0.49\}$ 
and compare its performance against two baselines: 
\textit{Standard} (where $\delta_t = 0$) and 
\textit{Uniform} (where $\delta_t = 1$). 
Each algorithm is executed $1000$ times, and we report the averaged results.

\paragraph{Results.} The simulation results are shown in Fig. \ref{fig2b}, \ref{fig2c} and \ref{fig2d}. From Fig. \ref{fig2b}, the Uniform baseline consistently exhibits the highest cumulative regret throughout the entire horizon. In contrast, both the Standard baseline and \texttt{EXP3-N-CS} with larger $\alpha$ values (e.g., $\alpha = 0.4$ or $\alpha = 0.49$) achieve the lowest cumulative regret. This is because Uniform does not focus on minimizing regret. Fig. \ref{fig2c} illustrates the trajectories of the CS width \( \hat{C}_t(S_i, S_j) \), where $S_i,S_j = \arg\max_{S_i,S_j \in \U_\bbE} \hat{C}_T(S_i,S_j)$ ($\hat{C}_T(S_i,S_j)$ takes the average value of 1000 times). 
The Uniform baseline achieves the narrowest CS, indicating the most accurate inference. 
In contrast, the Standard baseline maintains the widest CS width throughout the horizon, implying the least accurate inference. 
The \texttt{EXP3-N-CS} variants lie between these two extremes, with smaller $\alpha$ values producing wider CS widths that approach that of Uniform. Fig. \ref{fig2d} presents the box plot of the maximum ATE estimation error (i.e., $e_\nu(T,\hat{\Delta})$), where the orange line represents the median. As shown in Fig.~\ref{fig2d}, both \texttt{EXP3-N-CS} variants with smaller $\alpha$ values and the Uniform baseline achieve relatively low maximum ATE estimation errors with compact interquartile ranges and fewer extreme outliers. 
In contrast, the Standard baseline exhibits a noticeably wider spread of errors and a substantial number of outliers. 
This inferior inference performance of Standard (Obj.~2--3) is attributed to its lower frequency of exploring sub-optimal arms compared to \textit{Uniform} and the \texttt{EXP3-N-CS} variants. Due to page limitations, we present four extensive experimental instances in Section \ref{AExperiment} of the Appendix.

\section*{Acknowledgements}
Haoyang Hong and Huazheng Wang are supported by National Science Foundation under grant IIS-2403401.

\vskip 0.2in

\bibliographystyle{unsrtnat} 
\bibliography{sample}          

\begin{thebibliography}{57}
\providecommand{\natexlab}[1]{#1}
\providecommand{\url}[1]{\texttt{#1}}
\expandafter\ifx\csname urlstyle\endcsname\relax
  \providecommand{\doi}[1]{doi: #1}\else
  \providecommand{\doi}{doi: \begingroup \urlstyle{rm}\Url}\fi

\bibitem[Leung(2022{\natexlab{a}})]{leung2022causal}
Michael~P Leung.
\newblock Causal inference under approximate neighborhood interference.
\newblock \emph{Econometrica}, 90\penalty0 (1):\penalty0 267--293, 2022{\natexlab{a}}.

\bibitem[Leung(2022{\natexlab{b}})]{leung2022rate}
Michael~P Leung.
\newblock Rate-optimal cluster-randomized designs for spatial interference.
\newblock \emph{The Annals of Statistics}, 50\penalty0 (5):\penalty0 3064--3087, 2022{\natexlab{b}}.

\bibitem[Leung(2023)]{leung2023network}
Michael~P Leung.
\newblock Network cluster-robust inference.
\newblock \emph{Econometrica}, 91\penalty0 (2):\penalty0 641--667, 2023.

\bibitem[Imbens(2024)]{imbens2024causal}
Guido~W Imbens.
\newblock Causal inference in the social sciences.
\newblock \emph{Annual Review of Statistics and Its Application}, 11, 2024.

\bibitem[Arpino and Mattei(2016)]{arpino2016assessing}
Bruno Arpino and Alessandra Mattei.
\newblock Assessing the causal effects of financial aids to firms in tuscany allowing for interference.
\newblock \emph{The Annals of Applied Statistics}, 2016.

\bibitem[Munro et~al.(2021)Munro, Wager, and Xu]{munro2021treatment}
Evan Munro, Stefan Wager, and Kuang Xu.
\newblock Treatment effects in market equilibrium.
\newblock \emph{arXiv preprint arXiv:2109.11647}, 2021.

\bibitem[Bandiera et~al.(2009)Bandiera, Barankay, and Rasul]{bandiera2009social}
Oriana Bandiera, Iwan Barankay, and Imran Rasul.
\newblock Social connections and incentives in the workplace: Evidence from personnel data.
\newblock \emph{Econometrica}, 77\penalty0 (4):\penalty0 1047--1094, 2009.

\bibitem[Bond et~al.(2012)Bond, Fariss, Jones, Kramer, Marlow, Settle, and Fowler]{bond201261}
Robert~M Bond, Christopher~J Fariss, Jason~J Jones, Adam~DI Kramer, Cameron Marlow, Jaime~E Settle, and James~H Fowler.
\newblock A 61-million-person experiment in social influence and political mobilization.
\newblock \emph{Nature}, 489\penalty0 (7415):\penalty0 295--298, 2012.

\bibitem[Paluck et~al.(2016)Paluck, Shepherd, and Aronow]{paluck2016changing}
Elizabeth~Levy Paluck, Hana Shepherd, and Peter~M Aronow.
\newblock Changing climates of conflict: A social network experiment in 56 schools.
\newblock \emph{Proceedings of the National Academy of Sciences}, 113\penalty0 (3):\penalty0 566--571, 2016.

\bibitem[Agarwal et~al.(2024)Agarwal, Agarwal, Masoero, and Whitehouse]{agarwal2024multi}
Abhineet Agarwal, Anish Agarwal, Lorenzo Masoero, and Justin Whitehouse.
\newblock Multi-armed bandits with network interference.
\newblock \emph{arXiv preprint arXiv:2405.18621}, 2024.

\bibitem[Jia et~al.(2024)Jia, Frazier, and Kallus]{jia2024multi}
Su~Jia, Peter Frazier, and Nathan Kallus.
\newblock Multi-armed bandits with interference.
\newblock \emph{arXiv preprint arXiv:2402.01845}, 2024.

\bibitem[Zhang and Wang(2024)]{zhang2024online}
Zhiheng Zhang and Zichen Wang.
\newblock Online experimental design with estimation-regret trade-off under network interference.
\newblock \emph{arXiv preprint arXiv:2412.03727}, 2024.

\bibitem[Gao and Ding(2023)]{gao2023causal}
Mengsi Gao and Peng Ding.
\newblock Causal inference in network experiments: regression-based analysis and design-based properties.
\newblock \emph{arXiv preprint arXiv:2309.07476}, 2023.

\bibitem[Simchi-Levi and Wang(2024)]{simchi2023multi}
David Simchi-Levi and Chonghuan Wang.
\newblock Multi-armed bandit experimental design: Online decision-making and adaptive inference.
\newblock \emph{Management Science}, 2024.
\newblock \doi{10.1287/mnsc.2023.00492}.

\bibitem[Mok et~al.(2021)Mok, Ku, and Yuda]{mok2021managing}
Ka~Ho Mok, Yeun-Wen Ku, and Tauchid~Komara Yuda.
\newblock Managing the covid-19 pandemic crisis and changing welfare regimes, 2021.

\bibitem[Ham et~al.(2023)Ham, Bojinov, Lindon, and Tingley]{ham2023design}
Dae~Woong Ham, Iavor Bojinov, Michael Lindon, and Martin Tingley.
\newblock Design-based inference for multi-arm bandits.
\newblock \emph{arXiv preprint arXiv:2302.14136}, 2023.

\bibitem[Woong~Ham et~al.(2023)Woong~Ham, Lindon, Tingley, and Bojinov]{woong2023design}
Dae Woong~Ham, Michael Lindon, Martin Tingley, and Iavor Bojinov.
\newblock Design-based confidence sequences: A general approach to risk mitigation in online experimentation.
\newblock \emph{Harvard Business School Technology \& Operations Mgt. Unit Working Paper}, \penalty0 (23-070), 2023.

\bibitem[Liang and Bojinov(2023)]{liang2023experimental}
Biyonka Liang and Iavor Bojinov.
\newblock An experimental design for anytime-valid causal inference on multi-armed bandits.
\newblock \emph{arXiv preprint arXiv:2311.05794}, 2023.

\bibitem[Lindon and Malek(2022)]{lindon2022anytime}
Michael Lindon and Alan Malek.
\newblock Anytime-valid inference for multinomial count data.
\newblock \emph{Advances in Neural Information Processing Systems}, 35:\penalty0 2817--2831, 2022.

\bibitem[Waudby-Smith et~al.(2024)Waudby-Smith, Wu, Ramdas, Karampatziakis, and Mineiro]{waudby2024anytime}
Ian Waudby-Smith, Lili Wu, Aaditya Ramdas, Nikos Karampatziakis, and Paul Mineiro.
\newblock Anytime-valid off-policy inference for contextual bandits.
\newblock \emph{ACM/JMS Journal of Data Science}, 1\penalty0 (3):\penalty0 1--42, 2024.

\bibitem[Waudby-Smith et~al.(2021)Waudby-Smith, Arbour, Sinha, Kennedy, and Ramdas]{waudby2021time}
Ian Waudby-Smith, David Arbour, Ritwik Sinha, Edward~H Kennedy, and Aaditya Ramdas.
\newblock Time-uniform central limit theory and asymptotic confidence sequences.
\newblock \emph{arXiv preprint arXiv:2103.06476}, 2021.

\bibitem[Duan et~al.(2024)Duan, Ma, Jiang, and Xia]{duan2024regret}
Congyuan Duan, Wanteng Ma, Jiashuo Jiang, and Dong Xia.
\newblock Regret minimization and statistical inference in online decision making with high-dimensional covariates.
\newblock \emph{arXiv preprint arXiv:2411.06329}, 2024.

\bibitem[Xu et~al.(2024)Xu, Lu, and Song]{xu2024linear}
Yang Xu, Wenbin Lu, and Rui Song.
\newblock Linear contextual bandits with interference.
\newblock \emph{arXiv preprint arXiv:2409.15682}, 2024.

\bibitem[Zhang and Imai(2023)]{zhang2023individualized}
Yi~Zhang and Kosuke Imai.
\newblock Individualized policy evaluation and learning under clustered network interference.
\newblock \emph{arXiv preprint arXiv:2311.02467}, 2023.

\bibitem[Viviano et~al.(2023)Viviano, Lei, Imbens, Karrer, Schrijvers, and Shi]{viviano2023causal}
Davide Viviano, Lihua Lei, Guido Imbens, Brian Karrer, Okke Schrijvers, and Liang Shi.
\newblock Causal clustering: design of cluster experiments under network interference.
\newblock \emph{arXiv preprint arXiv:2310.14983}, 2023.

\bibitem[Zhao(2024)]{zhao2024simple}
Jinglong Zhao.
\newblock A simple formulation for causal clustering.
\newblock \emph{Available at SSRN 5008213}, 2024.

\bibitem[Luedtke and Van Der~Laan(2016)]{luedtke2016statistical}
Alexander~R Luedtke and Mark~J Van Der~Laan.
\newblock Statistical inference for the mean outcome under a possibly non-unique optimal treatment strategy.
\newblock \emph{Annals of statistics}, 44\penalty0 (2):\penalty0 713, 2016.

\bibitem[Dimakopoulou et~al.(2017)Dimakopoulou, Zhou, Athey, and Imbens]{dimakopoulou2017estimation}
Maria Dimakopoulou, Zhengyuan Zhou, Susan Athey, and Guido Imbens.
\newblock Estimation considerations in contextual bandits.
\newblock \emph{arXiv preprint arXiv:1711.07077}, 2017.

\bibitem[Dimakopoulou et~al.(2019)Dimakopoulou, Zhou, Athey, and Imbens]{dimakopoulou2019balanced}
Maria Dimakopoulou, Zhengyuan Zhou, Susan Athey, and Guido Imbens.
\newblock Balanced linear contextual bandits.
\newblock In \emph{Proceedings of the AAAI Conference on Artificial Intelligence}, volume~33, pages 3445--3453, 2019.

\bibitem[Zhang et~al.(2020)Zhang, Janson, and Murphy]{zhang2020inference}
Kelly Zhang, Lucas Janson, and Susan Murphy.
\newblock Inference for batched bandits.
\newblock \emph{Advances in Neural Information Processing Systems}, 33:\penalty0 9818--9829, 2020.

\bibitem[Dimakopoulou et~al.(2021)Dimakopoulou, Ren, and Zhou]{dimakopoulou2021online}
Maria Dimakopoulou, Zhimei Ren, and Zhengyuan Zhou.
\newblock Online multi-armed bandits with adaptive inference.
\newblock \emph{Advances in Neural Information Processing Systems}, 34:\penalty0 1939--1951, 2021.

\bibitem[Hadad et~al.(2021)Hadad, Hirshberg, Zhan, Wager, and Athey]{hadad2021confidence}
Vitor Hadad, David~A Hirshberg, Ruohan Zhan, Stefan Wager, and Susan Athey.
\newblock Confidence intervals for policy evaluation in adaptive experiments.
\newblock \emph{Proceedings of the national academy of sciences}, 118\penalty0 (15):\penalty0 ~e2014602118, 2021.

\bibitem[Zhang et~al.(2021)Zhang, Janson, and Murphy]{zhang2021statistical}
Kelly Zhang, Lucas Janson, and Susan Murphy.
\newblock Statistical inference with m-estimators on adaptively collected data.
\newblock \emph{Advances in Neural Information Processing Systems}, 34:\penalty0 7460--7471, 2021.

\bibitem[Han et~al.(2022)Han, Sun, and Zhang]{han2022online}
Qiyu Han, Will~Wei Sun, and Yichen Zhang.
\newblock Online statistical inference for matrix contextual bandit.
\newblock \emph{arXiv preprint arXiv:2212.11385}, 2022.

\bibitem[Deshpande et~al.(2023)Deshpande, Javanmard, and Mehrabi]{deshpande2023online}
Yash Deshpande, Adel Javanmard, and Mohammad Mehrabi.
\newblock Online debiasing for adaptively collected high-dimensional data with applications to time series analysis.
\newblock \emph{Journal of the American Statistical Association}, 118\penalty0 (542):\penalty0 1126--1139, 2023.

\bibitem[Auer et~al.(2002{\natexlab{a}})Auer, Cesa-Bianchi, Freund, and Schapire]{auer2002nonstochastic}
Peter Auer, Nicolo Cesa-Bianchi, Yoav Freund, and Robert~E Schapire.
\newblock The nonstochastic multiarmed bandit problem.
\newblock \emph{SIAM journal on computing}, 32\penalty0 (1):\penalty0 48--77, 2002{\natexlab{a}}.

\bibitem[Lattimore and Szepesv{\'a}ri(2020)]{lattimore2020bandit}
Tor Lattimore and Csaba Szepesv{\'a}ri.
\newblock \emph{Bandit algorithms}.
\newblock Cambridge University Press, 2020.

\bibitem[Auer et~al.(2002{\natexlab{b}})Auer, Cesa-Bianchi, and Fischer]{Auer2002FinitetimeAO}
Peter Auer, Nicol{\`o} Cesa-Bianchi, and Paul Fischer.
\newblock Finite-time analysis of the multiarmed bandit problem.
\newblock \emph{Machine Learning}, 47:\penalty0 235--256, 2002{\natexlab{b}}.

\bibitem[Darling and Robbins(1967)]{darling1967confidence}
Donald~A Darling and Herbert Robbins.
\newblock Confidence sequences for mean, variance, and median.
\newblock \emph{Proceedings of the National Academy of Sciences}, 58\penalty0 (1):\penalty0 66--68, 1967.

\bibitem[Besson and Kaufmann(2018{\natexlab{a}})]{besson2018doubling}
Lilian Besson and Emilie Kaufmann.
\newblock What doubling tricks can and can't do for multi-armed bandits.
\newblock \emph{arXiv preprint arXiv:1803.06971}, 2018{\natexlab{a}}.

\bibitem[Sz{\"o}r{\'e}nyi et~al.(2013)Sz{\"o}r{\'e}nyi, Busa-Fekete, Heged{\"u}s, Orm{\'a}ndi, Jelasity, and K{\'e}gl]{Szrnyi2013GossipbasedDS}
Bal{\'a}zs Sz{\"o}r{\'e}nyi, R{\'o}bert Busa-Fekete, Istv{\'a}n Heged{\"u}s, R{\'o}bert Orm{\'a}ndi, M{\'a}rk Jelasity, and Bal{\'a}zs K{\'e}gl.
\newblock Gossip-based distributed stochastic bandit algorithms.
\newblock In \emph{International Conference on Machine Learning}, 2013.

\bibitem[Wu et~al.(2016)Wu, Wang, Gu, and Wang]{Wu2016ContextualBI}
Qingyun Wu, Huazheng Wang, Quanquan Gu, and Hongning Wang.
\newblock Contextual bandits in a collaborative environment.
\newblock \emph{Proceedings of the 39th International ACM SIGIR conference on Research and Development in Information Retrieval}, 2016.

\bibitem[He et~al.(2022)He, Wang, Min, and Gu]{He2022ASA}
Jiafan He, Tianhao Wang, Yifei Min, and Quanquan Gu.
\newblock A simple and provably efficient algorithm for asynchronous federated contextual linear bandits.
\newblock \emph{arXiv preprint arXiv:2207.03106}, 2022.

\bibitem[Wang et~al.(2019)Wang, Hu, Chen, and Wang]{Wang2019DistributedBL}
Yuanhao Wang, Jiachen Hu, Xiaoyu Chen, and Liwei Wang.
\newblock Distributed bandit learning: Near-optimal regret with efficient communication.
\newblock \emph{arXiv preprint arxiv: 1904.06309}, 2019.

\bibitem[Wang et~al.(2023{\natexlab{a}})Wang, Li, Song, Wang, Gu, and Wang]{wang2023pure}
Zichen Wang, Chuanhao Li, Chenyu Song, Lianghui Wang, Quanquan Gu, and Huazheng Wang.
\newblock Pure exploration in asynchronous federated bandits.
\newblock \emph{arXiv preprint arXiv:2310.11015}, 2023{\natexlab{a}}.

\bibitem[Lou{\"e}dec et~al.(2015)Lou{\"e}dec, Chevalier, Mothe, Garivier, and Gerchinovitz]{louedec2015multiple}
Jonathan Lou{\"e}dec, Max Chevalier, Josiane Mothe, Aur{\'e}lien Garivier, and S{\'e}bastien Gerchinovitz.
\newblock A multiple-play bandit algorithm applied to recommender systems.
\newblock In \emph{The Twenty-Eighth International Flairs Conference}, 2015.

\bibitem[Lagr{\'e}e et~al.(2016)Lagr{\'e}e, Vernade, and Cappe]{lagree2016multiple}
Paul Lagr{\'e}e, Claire Vernade, and Olivier Cappe.
\newblock Multiple-play bandits in the position-based model.
\newblock \emph{Advances in Neural Information Processing Systems}, 29, 2016.

\bibitem[Zhou and Tomlin(2018)]{zhou2018budget}
Datong Zhou and Claire Tomlin.
\newblock Budget-constrained multi-armed bandits with multiple plays.
\newblock In \emph{Proceedings of the AAAI Conference on Artificial Intelligence}, volume~32, 2018.

\bibitem[Besson and Kaufmann(2018{\natexlab{b}})]{besson2018multi}
Lilian Besson and Emilie Kaufmann.
\newblock Multi-player bandits revisited.
\newblock In \emph{Algorithmic Learning Theory}, pages 56--92. PMLR, 2018{\natexlab{b}}.

\bibitem[Jia et~al.(2023)Jia, Oli, Anderson, Duff, Li, and Ravi]{jia2023short}
Su~Jia, Nishant Oli, Ian Anderson, Paul Duff, Andrew~A Li, and Ramamoorthi Ravi.
\newblock Short-lived high-volume bandits.
\newblock In \emph{International Conference on Machine Learning}, pages 14902--14929. PMLR, 2023.

\bibitem[Cesa-Bianchi and Lugosi(2012)]{cesa2012combinatorial}
Nicolo Cesa-Bianchi and G{\'a}bor Lugosi.
\newblock Combinatorial bandits.
\newblock \emph{Journal of Computer and System Sciences}, 78\penalty0 (5):\penalty0 1404--1422, 2012.

\bibitem[Chen et~al.(2013)Chen, Wang, and Yuan]{chen2013combinatorial}
Wei Chen, Yajun Wang, and Yang Yuan.
\newblock Combinatorial multi-armed bandit: General framework and applications.
\newblock In \emph{International conference on machine learning}, pages 151--159. PMLR, 2013.

\bibitem[Chen et~al.(2014)Chen, Lin, King, Lyu, and Chen]{chen2014combinatorial}
Shouyuan Chen, Tian Lin, Irwin King, Michael~R Lyu, and Wei Chen.
\newblock Combinatorial pure exploration of multi-armed bandits.
\newblock \emph{Advances in Neural Information Processing Systems}, 27, 2014.

\bibitem[Combes et~al.(2015)Combes, Talebi Mazraeh~Shahi, Proutiere, et~al.]{combes2015combinatorial}
Richard Combes, Mohammad~Sadegh Talebi Mazraeh~Shahi, Alexandre Proutiere, et~al.
\newblock Combinatorial bandits revisited.
\newblock \emph{Advances in Neural Information Processing Systems}, 28, 2015.

\bibitem[Kveton et~al.(2015)Kveton, Wen, Ashkan, and Szepesvari]{kveton2015combinatorial}
Branislav Kveton, Zheng Wen, Azin Ashkan, and Csaba Szepesvari.
\newblock Combinatorial cascading bandits.
\newblock \emph{Advances in Neural Information Processing Systems}, 28, 2015.

\bibitem[Saha and Gopalan(2019)]{saha2019combinatorial}
Aadirupa Saha and Aditya Gopalan.
\newblock Combinatorial bandits with relative feedback.
\newblock \emph{Advances in Neural Information Processing Systems}, 32, 2019.

\bibitem[Wang et~al.(2023{\natexlab{b}})Wang, Balasubramanian, Yuan, Song, Wang, and Wang]{wang2023adversarial}
Zichen Wang, Rishab Balasubramanian, Hui Yuan, Chenyu Song, Mengdi Wang, and Huazheng Wang.
\newblock Adversarial attacks on online learning to rank with stochastic click models.
\newblock \emph{arXiv preprint arXiv:2305.19218}, 2023{\natexlab{b}}.

\end{thebibliography}

\newpage
\appendix
\onecolumn

\section{Notations}

\begin{table}[h!]
\centering
\begin{tabular}{l|l}
\toprule
$\mathcal{U}$ & Set of units \\
$N$ & Number of units \\
$\bH$ & Adjacency matrix \\
$\mathcal{K}$ & Real arm set (action set) \\
$K$ & Number of real arms \\
$a_{i,t}$ & Arm assigned to unit $i$ \\
$A_t$ & Real super arm pulled in round $t$ \\
$\bbS(i,A,\bH)$ & Exposure mapping\\
$s_{i,t}$ & Exposure arm assigned to unit $i$ \\
$S_t$ & Exposure super arm sampled in round $t$ \\
$\hat{R}_{\mathcal{L}_m,t}(S)$ & Reward estimator for exposure super arm $S$ \\
$\U_s$ & Set of exposure super-arms\\
$d_s$ & Number of exposure arm\\
$\mathcal{U}_\mathcal{E}$ & Legitimate exposure super arm set  \\
$\mathcal{U}_\mathcal{O}$ & Set of exposure super arm that can be triggered by real super arm \\
$\mathcal{U}_\mathcal{C}$ & Set of cluster-wise switchback exposure super arm \\
$\pi_t^{\text{ALG}}(S)$ & Probability of Algorithm \texttt{ALG} pulling exposure super arm $S$ \\
$\pi_t^{\text{MAD}}(S)$ & Probability of pulling exposure super arm $S$ after using MAD \\
$\mathcal{E}_0$ & Set of legitimate instances \\
$Y_{i,t}(A)$ & Expected reward of the unit $i$ under $A$\\
$\tilde{Y}_{i,t}(S)$ & Expected reward of unit $i$ under $S$\\
$\bbY_t(S)$ & Average expected reward under $S$ \\
$R_t(S_t)$ & Average reward under $S_t$ in round $t$ \\
$\mathcal{R}(T, \pi)$ & Cumulative regret \\
$\tau_t(S_i,S_j)$ & Difference between potential outcome of $S_i$ and $S_j$ \\
$\bar{\tau}_t(S_i,S_j)$ & ATE between $S_i$ and $S_j$\\
$\hat{\tau}_t(S_i, S_j)$ & IPW estimator for $\tau_t(S_i,S_j)$ \\
$\hat{\bar{\tau}}_t(S_i, S_j)$ & IPW estimator for $\bar{\tau}_t(S_i,S_j)$ \\
$\hat{C}_t(S_i, S_j)$ & CS width \\
$\V_t(S_i,S_j)$ & Cumulative conditional variance between $S_i$ and $S_j$\\
$\hat{\V}_t(S_i,S_j)$ & Estimator of the cumulative conditional variance between $S_i$ and $S_j$ \\
$\{\bar{\tau}_t(S_i,S_j) \pm \hat{C}_t(S_i, S_j)\}_{t = 1}^\infty$ & Confidence sequence \\
$\hat{\Delta}_T(S_i, S_j)$ & Estimated ATE between $S_i$ and $S_j$ \\
$e_\nu(T, \hat{\Delta})$ & Maximum estimation error of the ATE \\
\toprule
\end{tabular}
\end{table}

\newpage

\section{Comparing MABNI with Multiple-Play, Multi-Agent and Combinatorial Bandits}

The MABNI problem is related to the \textit{multi-agent bandit} problem, in which multiple agents simultaneously pull arms in each round. These agents often collaborate by sharing their local observations to collectively accelerate learning \citep{Szrnyi2013GossipbasedDS,Wu2016ContextualBI,He2022ASA,Wang2019DistributedBL,wang2023pure}. A key distinction lies in the modeling assumptions: multi-agent bandit formulations typically assume a priori relationships among agents—such as cooperation or competition—and place significant emphasis on the design of communication protocols to enable coordination or negotiation.
Besides, the \textit{multi-play bandit} problem, where the algorithm selects multiple arms in each round and receives individual reward feedback for each, is closely related to the MABNI setting. This framework has been extensively studied in the literature \citep{louedec2015multiple,lagree2016multiple,zhou2018budget,besson2018multi,jia2023short}. While both settings involve the simultaneous selection of multiple actions, MABNI further emphasizes the interference among actions selected at different units, where the reward of a unit may depend not only on its own action but also on the other actions selected in the same round. Furthermore, our work is also related to the \textit{combinatorial bandit} problem, where the learner selects a subset of base arms—often subject to combinatorial constraints such as budgets or matroids—and receives feedback and rewards that depend on the selected combination \citep{cesa2012combinatorial,chen2013combinatorial,chen2014combinatorial,combes2015combinatorial,kveton2015combinatorial,saha2019combinatorial,wang2023adversarial}. Some existing works consider interference effects among units, but such interference is typically either explicitly known to the learner or assumed to follow a predefined structural pattern. In contrast, the MABNI makes no assumptions about the nature or structure of interference across units; instead, it needs to implicitly learn the interference effects through observed rewards

\section{Proof of Theorem \ref{trade-off}}

\begin{proof}[Proof of Theorem \ref{trade-off}]

Recall that the definition of ATE in round $t$ is defined as
\begin{align}
    \begin{split}
    \nonumber
        \bar{\tau}_t(S_i,S_j) &= \frac{1}{t}\sum_{t^\prime = 1}^t \tau_t(S_i,S_j) = \frac{1}{t}\sum_{t^\prime = 1}^t \big( \bbY_{t^\prime}(S_i) - \bbY_{t^\prime}(S_j) \big),
    \end{split}
    \end{align}
and the definition of regret is
\begin{align}
\mathcal{R}(T, \pi)
=
\max_{S \in \mathcal{U}_{\mathcal{E}}} \sum_{t=1}^T  \bbY_t(S) 
- \mathbb{E}_{\pi} \left[\sum_{t=1}^T R_t(S_t)\right].
\end{align}
Here  \begin{equation}
        \begin{aligned}
        \Tilde{Y}_{i,t}(S_{}) &= \sum_{A \in \mathcal{K}^{\U}} Y_{i,t}(A) \mathbb{P}(A_t = A\mid S),~\mathcal{Y}_t(S) = \frac{1}{N}\sum_{i \in \mathcal{U}} \tilde{Y}_{i,t}(S). 
\end{aligned}
\end{equation}
 Given a fixed policy $\pi$, we provide the following hard instances. We define the first instance as $\nu_1 \in \bbE_0$, in which $Y_{i,t}(A) \sim \text{Bernoulli}(f_i(A))$. We denote the best arm as $S^\prime$ and $S := \arg\min_{S \in \mathcal{U}_{\mathcal{E}},S\not=S'} \tilde{\bar{\tau}}_T^1(S,S^\prime) \mathbb{E}_{\nu_1}[\N^T_S]$, where $\N^T_S = \sum_{t=1}^T \bone\{S_t = S\}$ and $\tilde{\bar{\tau}}_T^1(S,S^\prime) := \frac{1}{N}\sum_{i \in \mathcal{U}} \sum_{A \in \mathcal{K}^{\U}} f_{i}(A) \big(\mathbb{P}(A_t = A\mid S)  
-  \mathbb{P}(A_t = A \mid S^\prime ) \big)$. 
The difference in treatment effect between $S$ and $S'$,
\[
\bar{\tau}_T^{\nu_1}(S, S') := \frac{1}{T}\frac{1}{N}\sum_{t=1}^T\sum_{i \in \mathcal{U}} \big(\Tilde{Y}_{i,t}(S) - \Tilde{Y}_{i,t}(S')\big),
\]
can be equivalently expressed as (for brevity, we use $\bar{\tau}_1$ to denote $\bar{\tau}_T^{\nu_1}(S, S')$ in the subsequent discussion)
\begin{equation}
\begin{aligned}
  &\bar{\tau}_1 = \frac{1}{T}\frac{1}{N}\sum_{t=1}^T\sum_{i \in \mathcal{U}} \sum_{A \in \mathcal{K}^{\U}} Y_{i,t}(A) \Big(\mathbb{P}(A_t = A\mid S)  
-  \mathbb{P}(A_t = A \mid S^\prime ) \Big).
\end{aligned}\label{delta_1}
\end{equation}
Based on the fact that $\frac{1}{N}\sum_{i \in \mathcal{U}} \sum_{A \in \mathcal{K}^{\U}} Y_{i,t}(A) \big(\mathbb{P}(A_t = A\mid S)  
-  \mathbb{P}(A_t = A \mid S^\prime ) \big)$
is $1$-sub-Gaussian, and for all $t\in[T]$
\[
\bE\!\left[\frac{1}{N}\sum_{i \in \mathcal{U}} \sum_{A \in \mathcal{K}^{\U}} Y_{i,t}(A) \Big(\mathbb{P}(A_t = A\mid S)  
-  \mathbb{P}(A_t = A \mid S^\prime ) \Big)\right] 
= \frac{1}{N}\sum_{i \in \mathcal{U}} \sum_{A \in \mathcal{K}^{\U}} f_{i}(A) \Big(\mathbb{P}(A_t = A\mid S)  
-  \mathbb{P}(A_t = A \mid S^\prime ) \Big),
\]
the Hoeffding inequality implies that, with probability at least $1-\tfrac{1}{T}$,
\begin{align}\label{concen1}
     \tilde{\bar{\tau}}_1 + \sqrt{\tfrac{2\log(2T)}{T}} \ge \bar{\tau}_1.
\end{align}
 On the other hand, we construct another instance as ($\beta\in (0,1)$ is chosen as sufficiently small)
  \begin{equation}
        \begin{aligned}
          Y^\prime_{i,t}(A) := \begin{cases}
                \text{Bernoulli}(f_i(A)) &\forall A\text{~satisfying~} \mathbb{P}(A_t = A \mid S ) = 0,   \\
               \text{Bernoulli}(f_i(A)-\beta) & \forall A\text{~satisfying~} \mathbb{P}(A_t = A \mid S )  > 0.
          \end{cases}
        \end{aligned}\label{construct}
        \end{equation}
 It leads to
\begin{equation}
\begin{aligned}
\nonumber
&\bar{\tau}_2 = \bar{\tau}_{2,=0} + \bar{\tau}_{2,>0}, \text{~where~}\\
&\bar{\tau}_{2,=0} := \frac{1}{T}\frac{1}{N}\sum_{t=1}^T \sum_{i \in \mathcal{U}} \sum_{A \in \mathcal{K}^{\U}} Y^\prime_{i,t}(A)  \Big(\mathbb{P}(A_t = A\mid S)  
-  \mathbb{P}(A_t = A \mid S^\prime ) \Big)\mathbf{1}\{\mathbb{P}( A_t = A \mid S ) = 0\},\\
&\bar{\tau}_{2,>0} := \frac{1}{T}\frac{1}{N}\sum_{t=1}^T \sum_{i \in \mathcal{U}} \sum_{A \in \mathcal{K}^{\U}} Y^\prime_{i,t}(A) \Big(\mathbb{P}(A_t = A\mid S)  
-  \mathbb{P}(A_t = A \mid S^\prime ) \Big)\mathbf{1}\{\mathbb{P}( A_t = A \mid S ) > 0\}.
\end{aligned}
\end{equation}
Follow the similar argument as Eq~\eqref{concen1}, we have with probability at least $1-\frac{1}{T}$
\begin{align}\label{concen2}
    \bar{\tau}_2 \ge \tilde{\bar{\tau}}_2 - \sqrt{\frac{2\log(2T)}{T}},
\end{align}
where
\begin{equation}
\begin{aligned}
\nonumber
&\tilde{\bar{\tau}}_2 = \tilde{\bar{\tau}}_{2,=0} + \tilde{\bar{\tau}}_{2,>0}, \text{~where~}\\
&\tilde{\bar{\tau}}_{2,=0} := \frac{1}{N} \sum_{i \in \mathcal{U}} \sum_{A \in \mathcal{K}^{\U}} f_{i}(A)  \Big(\mathbb{P}(A_t = A\mid S)  
-  \mathbb{P}(A_t = A \mid S^\prime ) \Big)\mathbf{1}\{\mathbb{P}( A_t = A \mid S ) = 0\},\\
&\tilde{\bar{\tau}}_{2,>0} := \frac{1}{N} \sum_{i \in \mathcal{U}} \sum_{A \in \mathcal{K}^{\U}} (f_{i}(A) - \beta) \Big(\mathbb{P}(A_t = A\mid S)  
-  \mathbb{P}(A_t = A \mid S^\prime ) \Big)\mathbf{1}\{\mathbb{P}( A_t = A \mid S ) > 0\}.
\end{aligned}
\end{equation}
Define event $\bbE_G :=\{    \tilde{\bar{\tau}}_1 + \sqrt{\tfrac{2\log(2T)}{T}} \ge \bar{\tau}_1, ; \bar{\tau}_2 \ge \tilde{\bar{\tau}}_2 - \sqrt{\frac{2\log(2T)}{T}} \}$. Under this event:
\begin{equation}
    \begin{aligned}
    \nonumber
\bar{\tau}_2 - \bar{\tau}_1 \ge & -2\sqrt{\frac{2\log(2T)}{T}} + \frac{1}{N} \sum_{i \in \mathcal{U}} \sum_{A \in \mathcal{K}^{\U}} (- \beta) \Big(\mathbb{P}(A_t = A\mid S)  
-  \mathbb{P}(A_t = A \mid S^\prime ) \Big)\mathbf{1}\{\mathbb{P}( A_t = A \mid S ) > 0\}\\
=& - 2\sqrt{\frac{2\log(2T)}{T}} -\frac{1}{N}\sum_{i \in \mathcal{U}} \sum_{A \in \mathcal{K}^{\U}}  \beta \Big(\mathbb{P}(A_t = A\mid S)   \Big)\mathbf{1}\{\mathbb{P}( A_t = A \mid S ) > 0\} \\ =& -2\sqrt{\frac{2\log(2T)}{T}} - \beta.
    \end{aligned}
\end{equation}

    On this basis, given any pre-specified estimator and strategy, which is recorded as $\{\hat{\Delta}_t\}_{ t \in [T]}$, following~\citet{zhang2024online, simchi2023multi}, we establish a hypothesis test as $\psi(\hat{\Delta}_T) = \arg\min_{i=1,2}|\hat{\Delta}_T - \bar{\tau}_i|$, implying that $\psi(\hat{\Delta}_T) \neq i, i\in \{1,2\}$ is a sufficient condition of $|\hat{\Delta}_T - \bar{\tau}_i| \geq \frac{1}{2}{\beta} + \sqrt{\frac{2\log(2T)}{T}}$. Therefore
\begin{equation}
    \begin{aligned}
    \inf _{\hat{\Delta}_T} \max _{\nu \in \mathcal{E}_0} \mathbb{P}_\nu\left(|\hat{\Delta}_T-\bar{\tau}_\nu| \geq \frac{1}{2}{\beta} + \sqrt{\frac{2\log(2T)}{T}} \right)  &\geq \inf _{\hat{\Delta}_T} \max _{i \in\{1,2\}} \mathbb{P}_{\nu_i}\left(|\hat{\Delta}_T-\bar{\tau}_i| \geq \frac{1}{2}{\beta} + \sqrt{\frac{2\log(2T)}{T}} \right) \\ &
 \geq \inf _{\hat{\Delta}_T} \max _{i \in\{1,2\}} \mathbb{P}_{\nu_i}\left(\psi(\hat{\Delta}_T ) \neq i\right) 
 \\ & \geq \inf _\psi \max _{i \in\{1,2\}} \mathbb{P}_{\nu_i}(\psi \neq i).
\end{aligned}\label{hypo}
\end{equation}
The above equation can directly lead to
\begin{equation}
    \begin{aligned}
    \text{RHS~of~}\eqref{hypo} \geq  \frac{1}{2}(1-\text{TV}(\mathbb{P}_{\nu_1}, \mathbb{P}_{\nu_2}))
\geq \frac{1}{2}\bigg[1-\sqrt{\frac{1}{2}\text{KL}(\mathbb{P}_{\nu_1}, \mathbb{P}_{\nu_2})}\bigg].
    \end{aligned}\label{tv_kl}
\end{equation}
Let $\bP_{\nu,S}(\cdot)$ denotes the reward density distribution conditioning on arm $S$ in $\nu$. Due to the fact that $\text{KL}(\mathbb{P}_{\nu_1}, \mathbb{P}_{\nu_2}) = \mathbb{E}_{\nu_1}[\N^T_S] \text{KL}( \mathbb{P}_{\nu_1,S}(\cdot), \mathbb{P}_{\nu_2,S}(\cdot))$, and 
\begin{equation}
\begin{aligned}
\text{KL}( \mathbb{P}_{\nu_1,S}(\cdot), \mathbb{P}_{\nu_2,S}(\cdot)) = \int_{ X} p_{\nu_1,S}(X) log\bigg(\frac{p_{\nu_1,S}(X)}{p_{\nu_2,S}(X)}\bigg)dX \leq q \beta^2 N,
\end{aligned}
\end{equation}
where $q > 0$ is a constant. It achieves that
\begin{equation}
\begin{aligned}
\text{KL}(\mathbb{P}_{\nu_1}, \mathbb{P}_{\nu_2}) & = \mathbb{E}_{\nu_1}[\N^T_S] \text{KL}( \mathbb{P}_{\nu_1,S}(\cdot), \mathbb{P}_{\nu_2,S}(\cdot))\\ & \leq q \beta^2 N \mathbb{E}_{\nu_1}[\N^T_S] \\ & \leq q \beta^2 N \frac{ \mathcal{R}^{stoc}_{\nu_1}(T, \pi)}{|\mathcal{U}_{\mathcal{E}}||\tilde{\bar{\tau}}_1|}\\
& \leq q \bigg(\beta + 2\sqrt{\frac{2\log(2T)}{T}} \bigg)^2 N \frac{ \mathcal{R}^{stoc}_{\nu_1}(T, \pi)}{|\mathcal{U}_{\mathcal{E}}||\tilde{\bar{\tau}}_1|},
\end{aligned}\label{kl_upper}
\end{equation}
where $\mathcal{R}^{\mathrm{stoc}}_{\nu_1}(\cdot)$ denotes the regret defined in the stochastic bandit setting under instance $\nu_1$. Here the last inequality is due to the definition of $S$. Combining~\eqref{tv_kl}-\eqref{kl_upper}, it implies
\begin{equation}
\begin{aligned}
    &\inf _{\hat{\Delta}_T} \max _{\nu \in \mathcal{E}_0} \mathbb{P}_\nu\left(\max_{S_i,S_j \in \U_{\mathcal{E}}}|\hat{\Delta}_T(S_i,S_j)-\bar{\tau}^\nu_T(S_i,S_j)| \geq \frac{\beta}{2} + \sqrt{\frac{2\log(2T)}{T}} \right) \\ \geq & \frac{1}{2}\bigg[1-\sqrt{\frac{1}{2}  q \Big(\beta + 2\sqrt{\frac{2\log(2T)}{T}}\Big)^2 N \frac{\mathcal{R}^{stoc}_{\nu_1}(T, \pi)}{|\mathcal{U}_{\mathcal{E}}| |\tilde{\bar{\tau}}_{1}|}}\bigg].
\end{aligned}\label{lower_estimation}
\end{equation}
Moreover, we aim to relate the regret in adversarial and stochastic settings. 
For any feasible stochastic instance $\nu$, obtained for example by Bernoulli sampling of $Y_{i,t}(A)$, we have
\begin{equation}\label{lower_regret}
    \mathcal{R}_{\nu}(T,\pi) \;\geq\; \mathcal{R}^{\mathrm{stoc}}_{\nu}(T,\pi),
\end{equation}
where the inequality follows from Jensen's inequality. Combining~\eqref{lower_estimation}-\eqref{lower_regret}, we get under event $\bbE_G$
\begin{equation}
    \begin{aligned}
      & \inf _{\hat{\Delta}_T} \max _{\nu \in \mathcal{E}_0} \mathbb{P}_\nu\left(\max_{S_i,S_j \in \U_{\mathcal{E}}}|\hat{\Delta}_T(S_i,S_j)-\bar{\tau}^\nu_T(S_i,S_j)| \geq \frac{\beta}{2} + \sqrt{\frac{2\log(2T)}{T}} \right)  \\ \geq & \frac{1}{2}\bigg[1-\sqrt{\frac{1}{2}  q \Big(\beta + 2\sqrt{\frac{2\log(2T)}{T}}\Big)^2 N \frac{\mathcal{R}_{\nu_1}(T, \pi)}{|\mathcal{U}_{\mathcal{E}}| |\tilde{\bar{\tau}}_{1}|}}\bigg].
    \end{aligned}
\end{equation}
As a consequence,
\begin{equation}
\begin{aligned}
&\inf _{\hat{\Delta}_T} \max _{\nu \in \mathcal{E}_0} \mathbb{E}_\nu\left(\max_{S_i,S_j \in \U_{\mathcal{E}}}|\hat{\Delta}_T(S_i,S_j)-\bar{\tau}^\nu_T(S_i,S_j)| \right) \sqrt{\mathcal{R}_{\nu_1}(T, \pi)} \\
\geq& 
 \frac{1}{2}\Big(1-\frac{2}{T}\Big) \bigg( \frac{\beta}{2} + \sqrt{\frac{2\log(2T)}{T}} \bigg)
 \bigg[1-\sqrt{\frac{1}{2}  q \Big(\beta + 2\sqrt{\frac{2\log(2T)}{T}}\Big)^2 N \frac{\mathcal{R}_{\nu_1}(T, \pi)}{|\mathcal{U}_{\mathcal{E}}| |\tilde{\bar{\tau}}_{1}|}}\bigg]\sqrt{\mathcal{R}_{\nu_1}(T, \pi)}.
\end{aligned}\label{final}
\end{equation}
When we choose $\beta$ such that $q \Big(\beta + 2\sqrt{\frac{\log(T/2)}{2T}}\Big)^2 N \frac{\mathcal{R}_{\nu_1}(T, \pi)}{|\mathcal{U}_{\mathcal{E}}| |\tilde{\bar{\tau}}_{1}|} = \frac{1}{2}$, it follows
    \begin{equation}
    \begin{aligned}
    \eqref{final} =&  
 \frac{1 }{8}\Big(1-\frac{2}{T}\Big)\sqrt{\frac{|\U_{\mathcal{E}}||\tilde{\bar{\tau}}_{1}|}{2 q N}} = \Omega_{K,T}(\sqrt{{|\U_{\mathcal{E}}|}{}}) .
    \end{aligned}
\end{equation}

\end{proof}

\section{Proof of Theorem \ref{CS}}

The following lemma is important in the proof of Theorem \ref{CS}:

\begin{lemma} \label{lemmacondition}
   Following the setting in Theorem \ref{CS}, for all $S_i, S_j \in \mathcal{U}_\bbE$, the sequence $\{\hat{\tau}_t(S_i, S_j)\}_{t=1}^\infty$ satisfies the Lindeberg-type uniform integrability condition (Condition L2 of Proposition 2.5) outlined by \citet{waudby2021time}, i.e., there exists $\beta \in (0,1)$ such that
\begin{align}
\nonumber
\sum_{t=1}^{\infty} \frac{\mathbb{E} \big[ \big(\hat{\tau}_t(S_i,S_j) - \tau_t(S_i,S_j)\big)^2 \mathbf{1}\big\{ \big(\hat{\tau}_t(S_i,S_j) - \tau_t(S_i,S_j)\big)^2 > \big(\V_t(S_i,S_j)\big)^{\beta} \big\} \big]}{\big(\V_t(S_i,S_j)\big)^{\beta}} < \infty \quad \text{a.s.,}
\end{align}
where $\V_t(S_i,S_j) = \sum_{t^\prime=1}^{t} \textbf{V}\big(\hat{\tau}_{t^\prime}(S_i,S_j) \mid \mathcal{F}_{t^\prime}\big)$ is the cumulative conditional variance.
\end{lemma}

\begin{proof}[Proof of Lemma \ref{lemmacondition}] We first upper bound $\big(\hat{\tau}_t(S_i, S_j) - \tau_t(S_i, S_j)\big)^2$. Based on the definition of our IPW estimator, we have
\begin{align}
    \begin{split}
    \nonumber
    &\big(\hat{\tau}_t(S_i,S_j) - \tau_t(S_i,S_j)\big)^2\\ =& \bigg( \frac{\bone\{S_t = S_i\}R_t(S_i)}{\pi_t^{\text{MAD}}(S_i)} - \frac{\bone\{S_t = S_j\}R_t(S_j)}{\pi_t^{\text{MAD}}(S_j)} - \tau_t(S_i,S_j) \bigg)^2\\
    \le& \frac{4}{\big(\pi_t^{\text{MAD}}(S_i)\wedge\pi_t^{\text{MAD}}(S_j)\big)^2} + \frac{8}{\big(\pi_t^{\text{MAD}}(S_i)\wedge\pi_t^{\text{MAD}}(S_j)\big)} + 4 \\
    \le& \frac{16}{\big(\pi_t^{\text{MAD}}(S_i)\wedge\pi_t^{\text{MAD}}(S_j)\big)^2},
    \end{split}
\end{align}    
where the first inequality is due to $R_t(S) \in [0,1]$. Note that, based on the setup of Theorem \ref{CS}, we have $\frac{1}{(\pi^{\text{MAD}}_t(S))^2} = O(t^{2\alpha})$ for all $S \in \mathcal{U}_\bbE$. This implies that $\big(\hat{\tau}_t(S_i, S_j) - \tau_t(S_i, S_j)\big)^2 = O(t^{2\alpha})$. Furthermore, based on Assumption \ref{assumption1}, we have $\V_t(S_i, S_j) = \Omega(t)$. Therefore, by setting $\beta \in \big(2\alpha, 1\big)$, there always exists a finite time $t^\prime$ such that for all $t \geq t^\prime$, $\big(\hat{\tau}_t(S_i, S_j) - \tau_t(S_i, S_j)\big)^2 \leq \big(\V_t(S_i, S_j)\big)^{\beta}$, and
\begin{align}
\begin{split}
\nonumber
&\sum_{t=1}^{\infty} \frac{\mathbb{E} \big[ \big(\hat{\tau}_t(S_i,S_j) - \tau_t(S_i,S_j)\big)^2 \mathbf{1}\big\{ \big(\hat{\tau}_t(S_i,S_j) - \tau_t(S_i,S_j)\big)^2 > \big(\V_t(S_i,S_j)\big)^{\beta} \big\} \big]}{\big(\V_t(S_i,S_j)\big)^{\beta}} \\ = & \sum_{t=1}^{t^\prime} \frac{\mathbb{E} \big[ \big(\hat{\tau}_t(S_i,S_j) - \tau_t(S_i,S_j)\big)^2 \mathbf{1}\big\{ \big(\hat{\tau}_t(S_i,S_j) - \tau_t(S_i,S_j)\big)^2 > \big(\V_t(S_i,S_j)\big)^{\beta} \big\} \big]}{\big(\V_t(S_i,S_j)\big)^{\beta}} \\ < & \infty \quad \text{a.s.}
\end{split}
\end{align}
Here we finish the proof of Lemma \ref{lemmacondition}.
\end{proof}

Based on Lemma \ref{lemmacondition}, we can prove Theorem \ref{CS}.

\begin{proof}[Proof of Theorem \ref{CS}]
Based on Assumption \ref{assumption1}, Lemma \ref{lemmacondition}, and Proposition 2.5 in \citet{waudby2021time}, $\{\hat{\bar{\tau}}_t(S_i, S_j) \pm C_t(S_i, S_j)\}_{t = 1}^\infty$ constitutes an asymptotic ($1 - \tilde{\delta}$) CS, where
    \begin{align}
    \nonumber
 C_t(S_i,S_j) = \sqrt{\frac{2\big( \V_t(S_i,S_j)\eta^2 + 1 \big)}{t^2\eta^2}\log\bigg( \frac{\sqrt{\V_t(S_i,S_j)\eta^2 + 1}}{\tilde{\delta}} \bigg)}.
    \end{align}
   Besides, based on the definition of the variance, we know that $\V_t \le \tilde{\V}_t$, where $\tilde{\V}_t = \sum_{t^\prime = 1}^t \sigma_{t^\prime}^2(S_i, S_j) = \sum_{t^\prime = 1}^t \Big(\frac{(\bbY_t(S_i))^2}{\pi^{\text{MAD}}_{t^\prime}(S_i)} + \frac{(\bbY_t(S_j))^2}{\pi^{\text{MAD}}_{t^\prime}(S_j)}\Big)$. Therefore, $\{\hat{\bar{\tau}}_t(S_i, S_j) \pm \tilde{C}_t(S_i, S_j)\}_{t = 1}^\infty$ is also an asymptotic ($1 - \tilde{\delta}$) CS, where
    \begin{align}
    \nonumber
 C_t(S_i,S_j) \le \tilde{C}_t(S_i,S_j) = \sqrt{\frac{2\big( \tilde{\V}_t(S_i,S_j)\eta^2 + 1 \big)}{t^2\eta^2}\log\bigg( \frac{\sqrt{\tilde{\V}_t(S_i,S_j)\eta^2 + 1}}{\tilde{\delta}} \bigg)}.
    \end{align}

Define $\hat{\sigma}_{t}^2(S_i, S_j) = \bigg( \frac{1}{\pi^{\text{MAD}}_{t}(S_i)} + \frac{1}{\pi^{\text{MAD}}_{t}(S_j)} \bigg)$ as the estimator of $\sigma_t^2(S_i, S_j)$, and let $\hat{\V}_t = \sum_{t^\prime = 1}^t \hat{\sigma}_{t^\prime}^2(S_i, S_j)$. Since $\hat{\V}_t \ge \tilde{\V}_t$, the sequence 
$\{\hat{\bar{\tau}}_t(S_i, S_j) \pm \hat{C}_t(S_i, S_j)\}_{t=1}^\infty$ 
forms an asymptotic $(1 - \tilde{\delta})$ confidence sequence, where
\begin{align}
\nonumber
   \hat{C}_t(S_i,S_j) = \sqrt{\frac{2\big( \hat{\V}_t(S_i,S_j)\eta^2 + 1 \big)}{t^2\eta^2}\log\bigg( \frac{\sqrt{\hat{\V}_t(S_i,S_j)\eta^2 + 1}}{\tilde{\delta}} \bigg)}.
\end{align}

We finally show that $\hat{C}_t(S_i,S_j) = \tilde{O}\big(\vert \U_\bbE \vert^{\frac{1}{2}} t^{\frac{\alpha - 1}{2}} \big)$ for all $S_i, S_j \in \U_\bbE$. We first upper bound $\hat{\V}_t(S_i,S_j)$, i.e., 
\begin{align}
\begin{split}
\nonumber
    \hat{\V}_t(S_i,S_j) & = \sum_{t^\prime = 1}^t \bigg( \frac{1}{\pi^{\text{MAD}}_{t}(S_i)} + \frac{1}{\pi^{\text{MAD}}_{t}(S_j)} \bigg) \\
    & \le \sum_{t^\prime = 1}^t (2\vert \U_\bbE \vert{t^{\prime}}^{\alpha}) \\
    & = O\big(\vert \U_\bbE \vert{t^{1 + \alpha}}\big).
\end{split}
\end{align}
Then 
\begin{align}
\nonumber
\hat{C}_t(S_i,S_j) = O\Bigg( \sqrt{\frac{\big(\vert \U_\bbE \vert t^{1 + \alpha}\eta^2 + 1 \big)}{t^2\eta^2}\log\bigg( \frac{\sqrt{\vert \U_\bbE \vert t^{1 + \alpha}\eta^2 + 1}}{\tilde{\delta}} \bigg)} \Bigg) = \tilde{O}\big(\vert \U_\bbE \vert^{\frac{1}{2}} t^{\frac{\alpha - 1}{2}} \big),
\end{align}
and it will converge to $0$ when $t \rightarrow \infty$. This concludes the proof of Theorem \ref{CS}.
\end{proof}

\section{Proof of Theorem \ref{ATEbound}}

\begin{proof}[Proof of Theorem \ref{ATEbound}, Claim (i)]
    Based on the result in Theorem \ref{CS}, for all $S_i \not = S_j$, with probability at least $1 - \tilde{\delta}$, we have
    \begin{align}
    \begin{split}
    \nonumber
       \vert \hat{\Delta}^{(i,j)} - \Delta^{(i,j)} \vert &\le 2\hat{C}_T(S_i, S_j)\\
       &= 2\sqrt{\frac{2\big( \hat{\V}_T(S_i,S_j)\eta^2 + 1 \big)}{T^2\eta^2}\log\bigg( \frac{\sqrt{\hat{\V}_T(S_i,S_j)\eta^2 + 1}}{\tilde{\delta}} \bigg)} \\
       & = \tilde{O}\Big(\vert\U_\bbE\vert T^{\alpha - \frac{1}{2}}\Big),
    \end{split}
    \end{align}
    where the first inequality is owing to Theorem \ref{CS}, and the last inequality is owing to the definition of $\hat{\V}_T(S_i,S_j)$. Finally, set $\tilde{\delta} = 1/T$, we have
    \begin{align}
    \nonumber
    \mathbb{E}[\vert \hat{\Delta}^{(i,j)} - \Delta^{(i,j)} \vert] \le 2(1 - \tilde{\delta})\hat{C}_t(S_i, S_j) + \tilde{\delta} T = \tilde{O}\Big(\vert\U_\bbE\vert T^{\alpha - \frac{1}{2}}\Big).
    \end{align}
\end{proof}

\begin{proof}[Proof of Theorem \ref{ATEbound} Claim (ii)] 
Let $\bE_{\text{MAD}}[\cdot]$ and $\bE_{\text{ALG}}[\cdot]$ denote the expectations taken with respect to the MAD and ALG (EXP3), respectively. Recall the definition of regret:
\begin{align}
    \R(T, \pi^{\text{MAD}}) = \max_{S\in\U_\bbE} \sum_{t=1}^T \bbY_t(S) - \mathbb{E}_{\text{MAD}}\left[\sum_{t=1}^T R_t(S_t)\right],
\end{align}
we also define
\begin{align}\label{RTI}
\begin{split}
\R(T, \pi^{\text{MAD}}, i) = \sum_{t=1}^T \bbY_t(S_i) - \mathbb{E}_{\text{MAD}}\left[\sum_{t=1}^T R_t(S_t)\right].
\end{split}
\end{align}
As the ''regret" assuming a fixed super arm $S_i$ is optimal for all $T$ rounds, while $\R(T, \pi)$ measures the actual regret relative to the best super arm at each round. If we can establish that $\R(T,\pi^{\text{MAD}},i) = \tilde{O}(\sqrt{\vert \U_\bbE \vert T} + T^{1 - \alpha})$ for all $S_i \in \U_\bbE$, it follows directly that $\R(T,\pi^{\text{MAD}}) = \tilde{O}(\sqrt{\vert \U_\bbE \vert T} + T^{1 - \alpha})$.

Based on the definition of the MAD, we can decompose Eq~\eqref{RTI} as
\begin{align}
\begin{split}
    \R(T, \pi^{\text{MAD}}, i) &= \sum_{t=1}^T \bbY_t(S_i) - \mathbb{E}_{\text{MAD}}\left[\sum_{t=1}^T R_t(S_t)\right]\\
    &= \sum_{t=1}^T \bbY_t(S_i) - \sum_{t = 1}^T  \Bigg( \delta_t \bigg( \frac{\sum_{S\in\U_\bbE}\bbY_t(S)}{\vert \U_\bbE \vert} \bigg) + (1 - \delta_t) \bE_{\text{ALG}}[R_t(S_t)] \Bigg)\\
    & = \R(T,\pi^{\text{ALG}},i) + \sum_{t = 1}^T \delta_t \bigg( \bE_{\text{ALG}}[R_t(S_t)] - \frac{\sum_{S\in\U_\bbE}\bbY_t(S)}{\vert \U_\bbE \vert} \bigg)\\
    & \le \R(T,\pi^{\text{ALG}},i) + 2T^{1 - \alpha},
\end{split}
\end{align}
where the third inequality is owing to the definition that $\R(T,\pi^{\text{ALG}},i) = \sum_{t=1}^T \bbY_t(S_i) - \sum_{t = 1}^T \bE_{\text{ALG}}[R_t(S_t)]$. We further decompose $\R(T, \pi^{\text{ALG}}, i)$. Let $M$ be such that $T \in \mathcal{L}_M$, and define $\R(\mathcal{L}_m, \pi^{\text{ALG}}, i)$ as $\R(\mathcal{L}_m, \pi^{\text{ALG}}, i) = \sum_{t \in \mathcal{L}_m} \bbY_t(S_i) - \mathbb{E}_{\text{MAD}}\left[\sum_{t \in \mathcal{L}_m} R_t(S_t)\right]$. It follows directly that $\R(T, \pi^{\text{ALG}}, i) \leq \sum_{m=1}^M \R(\mathcal{L}_m, \pi^{\text{ALG}}, i)$.

We now focusing on upper bound $\R(\mathcal{L}_m,\pi^{\text{ALG}},i)$. Set $\hat{R}_{\mathcal{L}_m,t_m-1}(S) = 0$ for all $S\in\U_\bbE$. Based on the unbiasedness of the IPW estimator, we have:
\begin{align}\label{IPWPROPER}
\begin{split}
\mathbb{E}_{\text{ALG}}[\hat{R}_{\mathcal{L}_m,t + 2^{m-1} - 1}(S)] &= \sum_{t \in \mathcal{L}_m} \bbY_t(S),\  \forall S \in \U_\bbE, \text{and} \\ \mathbb{E}_{\text{ALG}}\Big[ R_t(S_t) \Big\vert \bbH_{t-1} \Big] & = \sum_{S\in\U_\bbE} \pi^{\text{ALG}}_t(S) \bbY_t(S) \\ &= \sum_{S \in \U_\bbE} \pi^{\text{ALG}}_t(S) \mathbb{E}_{\text{ALG}}\Big[ \hat{R}_{\mathcal{L}_m,t}(S) - \hat{R}_{\mathcal{L}_m,t-1}(S) \Big\vert \bbH_{t-1}\Big],\ \forall t \in \mathcal{L}_m.
\end{split}
\end{align}
According to Eq~\eqref{IPWPROPER}, Eq~\eqref{RTI} can be rewritten as
\begin{align}
\begin{split}
\nonumber
    \R(\mathcal{L}_m,\pi^\text{ALG},i) & = \bE_{\text{ALG}} [\hat{R}_{\mathcal{L}_m,t_m + 2^{m-1} - 1}(S_i)] - \bE_{\text{ALG}} \bigg[ \sum_{t\in\mathcal{L}_m} R_t(S_t) \bigg]\\    & = \bE_{\text{ALG}} [\hat{R}_{\mathcal{L}_m,t_m + 2^{m-1} - 1}(S_i)] - \bE_{\text{ALG}} \Bigg[ \bE_{\text{ALG}} \bigg[ \sum_{t\in\mathcal{L}_m} R_t(S_t) \bigg\vert \bbH_{t-1} \bigg] \Bigg] \\
    & = \bE_{\text{ALG}} [\hat{R}_{\mathcal{L}_m,t_m + 2^{m-1} - 1}(S_i)] - \bE_{\text{ALG}} \Bigg[ \sum_{t \in \mathcal{L}_m} \sum_{S \in \U_\bbE} \pi^{\text{ALG}}_t(S)  \bE_{\text{ALG}} \bigg[\Big(\hat{R}_{\mathcal{L}_m,t}(S) - \hat{R}_{\mathcal{L}_m,t-1}(S)\Big) \Big\vert \bbH_{t-1} \bigg] \Bigg]\\
    & = \bE_{\text{ALG}} \Bigg[ \hat{R}_{\mathcal{L}_m,t_m + 2^{m-1} - 1}(S_i) - \sum_{t = 1}^T \sum_{S \in \U_\bbE} \pi^{\text{ALG}}_t(S)  \Big(\hat{R}_{\mathcal{L}_m,t}(S) - \hat{R}_{\mathcal{L}_m,t-1}(S)\Big) \Bigg]\\
    & = \bE_{\text{ALG}} \big[ \hat{R}_{\mathcal{L}_m,t_m + 2^{m-1} - 1}(S_i) - \hat{R}_{\mathcal{L}_m} \big],
\end{split}
\end{align}
where the first and third equalities follow from the tower rule, while the last equality holds due to our definition: $\hat{R}_{\mathcal{L}_m} = \sum_{t \in \mathcal{L}_m} \sum_{S \in \U_\bbE} \pi^{\text{ALG}}_t(S)  \big(\hat{R}_{\mathcal{L}_m,t}(S) - \hat{R}_{\mathcal{L}_m,t-1}(S)\big)$.

For $t \in \mathcal{L}_m$, we define $W_t = \sum_{S \in \U_\bbE} \exp\big(\epsilon_m \hat{R}_{\mathcal{L}_m,t}(S)\big)$. Consider the ratio between successive \(W_t\) and \(W_{t-1}\): $\frac{W_t}{W_{t-1}}$.
Using the definition of:  
\begin{align}
\pi^{\text{ALG}}_t(S) = \frac{\exp\big(\epsilon_m \hat{R}_{\mathcal{L}_m,t-1}(S)\big)}{W_{t-1}}.
\end{align}  
We rewrite the ratio as:
\begin{align}\label{eq19}
\begin{split}
\frac{W_t}{W_{t-1}} =& \sum_{S \in \U_\bbE} \pi_t^{\text{ALG}}(S) \exp\left(\epsilon_{m} \big(\hat{R}_{\mathcal{L}_m,t}(S) - \hat{R}_{\mathcal{L}_m,t-1}(S)\big)\right).
\end{split}
\end{align}

We now introduce two inequalities: 1) $\exp(x) \leq 1 + x + x^2$, $\forall x \leq 1$ ; 2) $1 + x \leq \exp(x)$, $\forall x > 0$. Based on these two inequalities, we can rewrite Eq~\eqref{eq19} as:
\begin{align}
\begin{split}
\nonumber
    & \sum_{S \in \U_\bbE} \pi_t^{\text{ALG}}(S) \exp\left(\epsilon_{m} \big(\hat{R}_{\mathcal{L}_m,t}(S) - \hat{R}_{\mathcal{L}_m,t-1}(S)\big)\right)\\
    \le& \bigg( 1 + \epsilon_{m} \sum_{S \in \U_\bbE} \pi_t^{\text{ALG}}(S)  \Big(\hat{R}_{\mathcal{L}_m,t}(S) - \hat{R}_{\mathcal{L}_m,t-1}(S)\Big)  + {\epsilon_{m}}^2 \sum_{S \in \U_\bbE} \pi_t^{\text{ALG}}(S)  \Big(\hat{R}_{\mathcal{L}_m,t}(S) - \hat{R}_{\mathcal{L}_m,t-1}(S)\Big)^2 \bigg)\\
    \le & \text{exp}\bigg( \epsilon_{m} \sum_{S \in \U_\bbE} \pi_t^{\text{ALG}}(S)  \Big(\hat{R}_{\mathcal{L}_m,t}(S) - \hat{R}_{\mathcal{L}_m,t-1}(S)\Big) + {\epsilon_{m}}^2 \sum_{S \in \U_\bbE} \pi^{\text{ALG}}_{t}(S)  \Big(\hat{R}_{\mathcal{L}_m,t}(S) - \hat{R}_{\mathcal{L}_m,t-1}(S) \Big)^2 \bigg).\\
\end{split}
\end{align}

Multiplying these ratios from \(t_m\) to \(t_m + 2^{m-1} -1\), we obtain:
\begin{align}
\nonumber
W_{t_m + 2^{m-1} -1} = \vert \U_\bbE \vert \prod_{t\in\mathcal{L}_m} \frac{W_t}{W_{t-1}} \leq \vert \U_\bbE \vert  \exp\left(\epsilon_{m} \hat{R}_{\mathcal{L}_m} +  {\epsilon_{m}}^2 \sum_{t\in\mathcal{L}_m} \sum_{S \in \U_\bbE} \pi^{\text{ALG}}_t(S)\Big(\hat{R}_{\mathcal{L}_m,t}(S) - \hat{R}_{\mathcal{L}_m,t-1}(S) \Big)^2\right).
\end{align}

Taking logarithms and rearranging the above equation, it yields:
\begin{align}
\nonumber
\hat{R}_{\mathcal{L}_m,t_m + 2^{m-1} - 1}(S_i)-\hat{R}_{\mathcal{L}_m} \le \frac{\log\big(\vert \U_\bbE \vert \big)}{\epsilon_m} + \epsilon_m\sum_{t\in\mathcal{L}_m} \sum_{S \in \U_\bbE} \pi^{\text{ALG}}_t(S)\Big(\hat{R}_{\mathcal{L}_m,t}(S) - \hat{R}_{\mathcal{L}_m,t-1}(S) \Big)^2.
\end{align}

Recalling the definition \( \R(\mathcal{L}_m, \pi^\text{ALG}, i) = \bE_{\text{{ALG}}}\big[ \hat{R}_{\mathcal{L}_m,t_m + 2^{m-1} - 1}(S_i) - \hat{R}_{\mathcal{L}_m} \big] \), we obtain:
\begin{align}
\nonumber
\R(\mathcal{L}_m, \pi^\text{ALG}, i) \leq \frac{\log\big(\vert \U_\bbE \vert \big)}{\epsilon_m} + \mathbb{E}_{\text{ALG}}\left[\epsilon_m\sum_{t\in\mathcal{L}_m} \sum_{S \in \U_\bbE} \pi^{\text{ALG}}_t(S)\Big(\hat{R}_{\mathcal{L}_m,t}(S) - \hat{R}_{\mathcal{L}_m,t-1}(S) \Big)^2\right].
\end{align}

We then try to bound $\mathbb{E}_{\text{ALG}}\left[\epsilon_m\sum_{t\in\mathcal{L}_m} \sum_{S \in \U_\bbE} \pi^{\text{ALG}}_t(S)\Big(\hat{R}_{\mathcal{L}_m,t}(S) - \hat{R}_{\mathcal{L}_m,t-1}(S) \Big)^2\right]$, there is 
\begin{align}
    \begin{split}
    \nonumber
       &\mathbb{E}_{\text{ALG}}\left[\epsilon_m\sum_{t\in\mathcal{L}_m} \sum_{S \in \U_\bbE} \pi^{\text{ALG}}_t(S)\Big(\hat{R}_{\mathcal{L}_m,t}(S) - \hat{R}_{\mathcal{L}_m,t-1}(S) \Big)^2\right]  \\ = & \bE_{\text{ALG}}\Bigg[ \epsilon_m \sum_{t \in \mathcal{L}_m}\sum_{S \in \U_\bbE} \pi^{\text{ALG}}_{t}(S) \bigg( 1 - \frac{\bone\{ S_t = S\} \big(1 - R_t(S)\big) }{\pi^{\text{ALG}}_{t}(S)} \bigg)^2\Bigg]\\
       = & \bE_{\text{ALG}}\bigg[ \epsilon_m \sum_{t \in \mathcal{L}_m}\sum_{S \in \U_\bbE} \pi^{\text{ALG}}_{t}(S) \Big( 1 - \frac{2 \times \bone\{ S_t = S\} \big( 1 - R_t(S) \big) }{\pi^{\text{ALG}}_{t}(S)} + \frac{\bone\{ S_t = S\} \big( 1 - R_t(S) \big)^2 }{\pi^{\text{ALG}}_{t}(S)^2} \Big)\bigg]\\
       = & \bE_{\text{ALG}}\Bigg[ \epsilon_m \sum_{t \in \mathcal{L}_m} \Big( 2R_t(S_t) \Big) + \bE_{\text{ALG}} \bigg[ \epsilon_m\sum_{t \in \mathcal{L}_m}\sum_{S \in \U_\bbE} \pi^{\text{ALG}}_{t}(S) \Big( \frac{\bone\{ S_t = S\} \big( 1 - R_t(S)\big)^2 }{\pi^{\text{ALG}}_{t}(S)^2} \Big) \bigg\vert \bbH_{t - 1} \bigg]\Bigg]\\
       = & \bE_{\text{ALG}}\bigg[ \epsilon_m \sum_{t \in \mathcal{L}_m} \Big(  2R_t(S_t) - 1 \Big) +  \epsilon_m\sum_{t \in \mathcal{L}_m} \sum_{S \in \U_\bbE} \big(1 - R_t(S)\big)^2  \bigg]\\
       \le & \vert \U_\bbE \vert 2^{m-1} \epsilon_m.
    \end{split}
\end{align}
Based on the definition of $\epsilon_m$, we conclude that $\R(\mathcal{L}_m, \pi^\text{ALG}, i) = \tilde{O}(\sqrt{\vert \U_\bbE \vert 2^{m-1}})$. We can upper bound $\R(T, \pi^\text{ALG}, i)$ by $\sum_{m=1}^M \R(\mathcal{L}_m, \pi^\text{ALG}, i) = \tilde{O}\left(\sum_{m=1}^M \sqrt{\vert \U_\bbE \vert 2^{m-1}} \right) = \tilde{O}(\sqrt{\vert \U_\bbE \vert} 2^{M/2})$. Owing to $M \le \log_2(T) + 1$, we have $\R(T, \pi^\text{ALG}, i) = \tilde{O}(\sqrt{\vert \U_\bbE \vert T})$. We can finally bound $\R(T, \pi^\text{MAD}, i)$ and $\R(T, \pi^\text{MAD})$ by $\tilde{O}(\sqrt{\vert \U_\bbE \vert T} + T^{1 - \alpha})$. 
\end{proof}

\begin{figure*}[t]
\centering
 \subfigbottomskip=4pt
	\subfigcapskip=-5pt 
	\subfigure[Network topology]{
\includegraphics[width=0.28\linewidth]{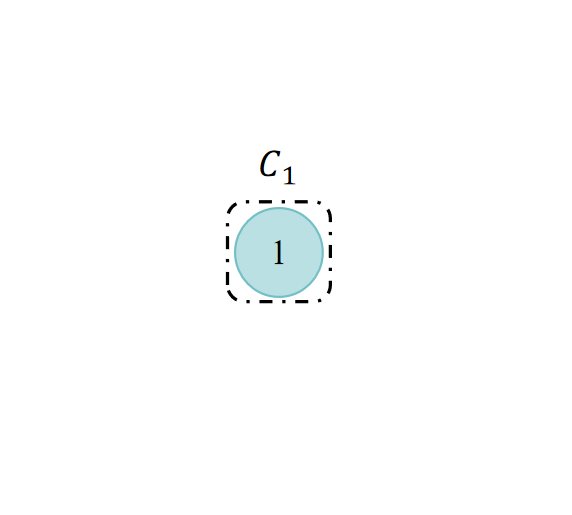}\label{fig3a}}\quad\quad
	\subfigure[Cumulative regret]{
\includegraphics[width=0.42\linewidth]{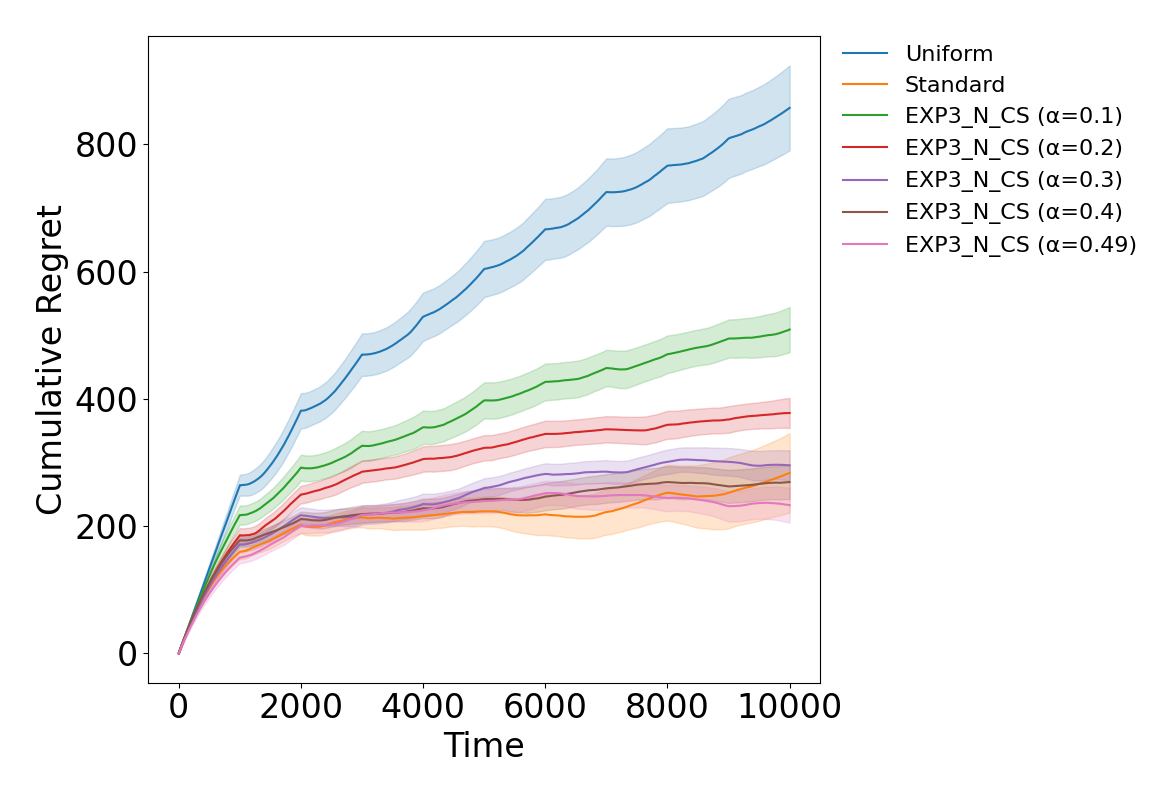}\label{fig3b}}\\
         \subfigure[CS width]{
\includegraphics[width=0.42\linewidth]{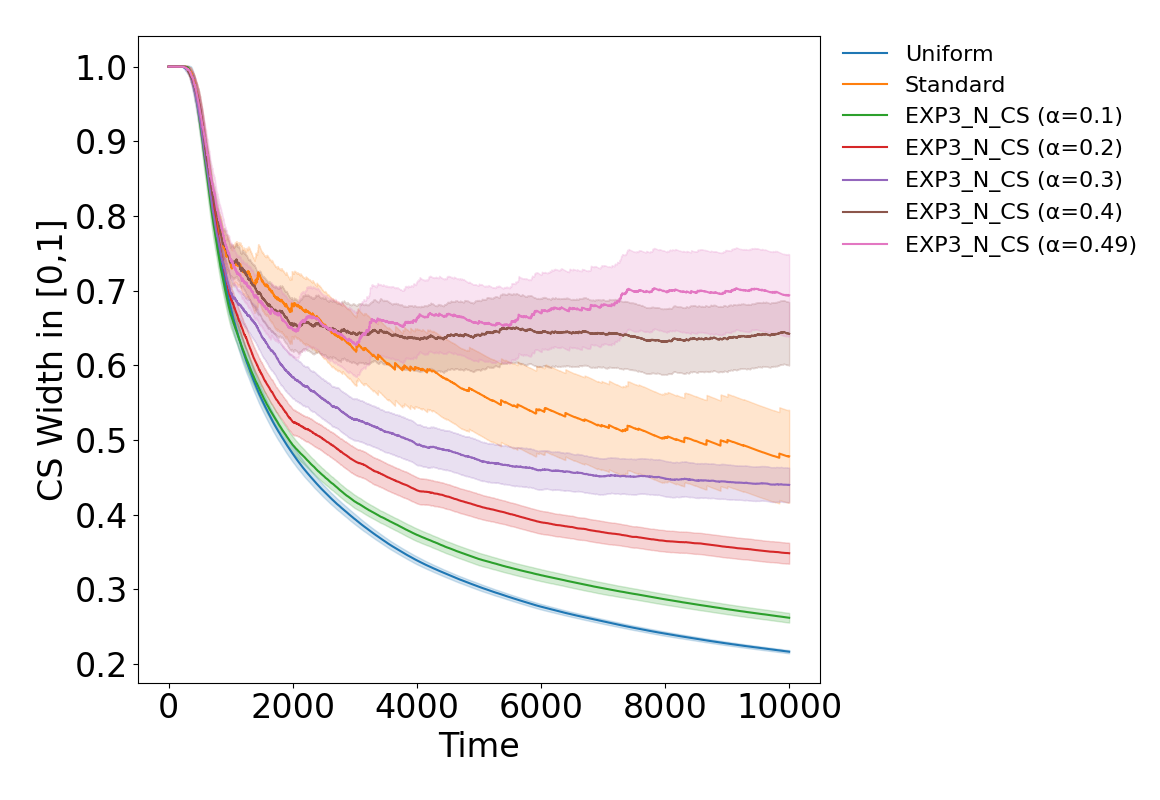}\label{fig3c}}
\subfigure[Maximum ATE estimation error]{
\includegraphics[width=0.45\linewidth]{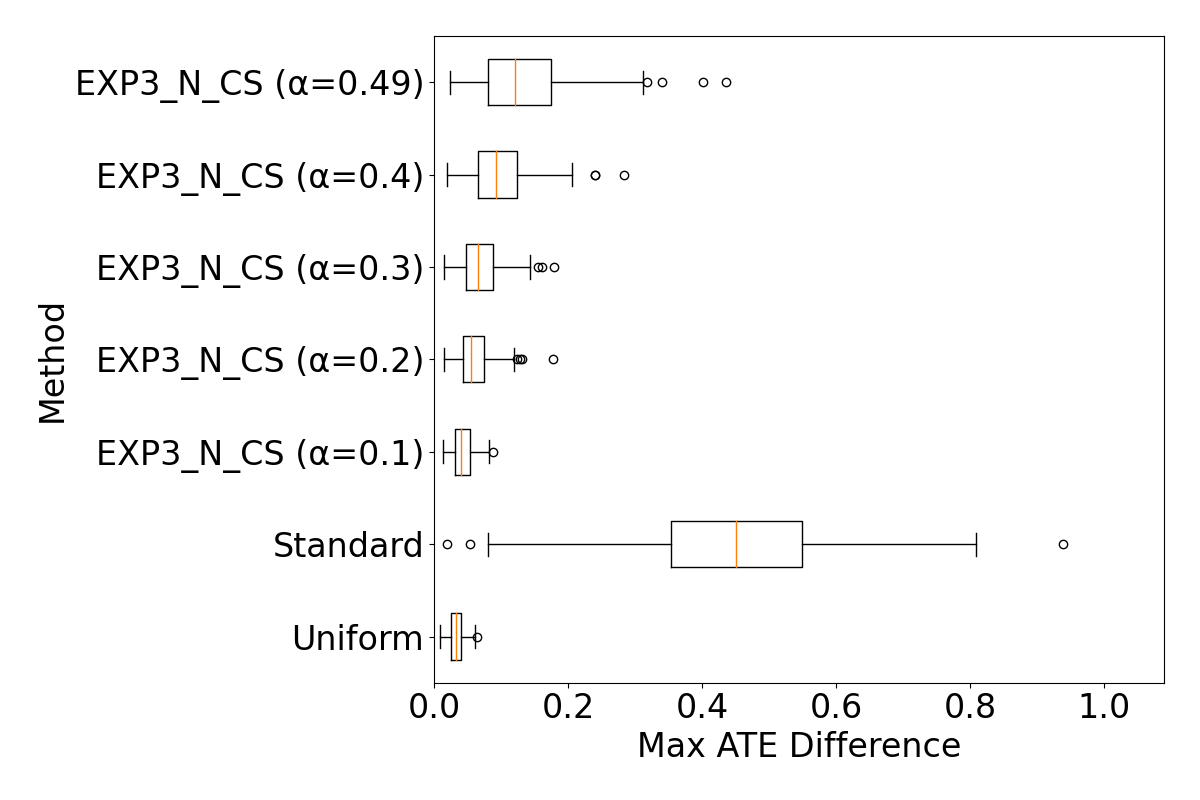}\label{fig3d}}
\caption{Experimental results of instance 1.}\label{fig3}
\end{figure*}

\begin{figure}[t]
\centering
 \subfigbottomskip=4pt
	\subfigcapskip=-5pt 
	\subfigure[Network topology]{
\includegraphics[width=0.28\linewidth]{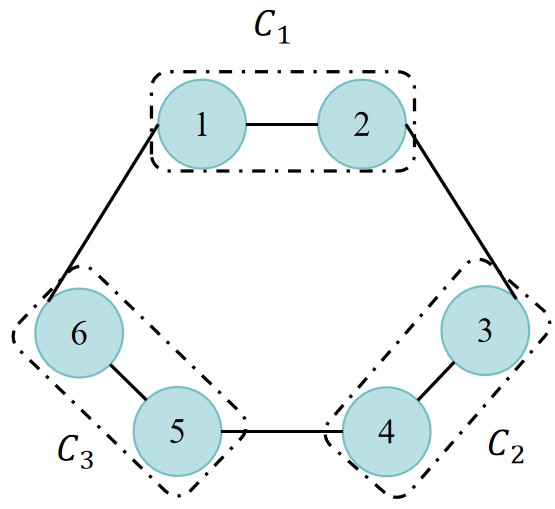}\label{fig4a}}\quad\quad
	\subfigure[Cumulative regret]{
\includegraphics[width=0.42\linewidth]{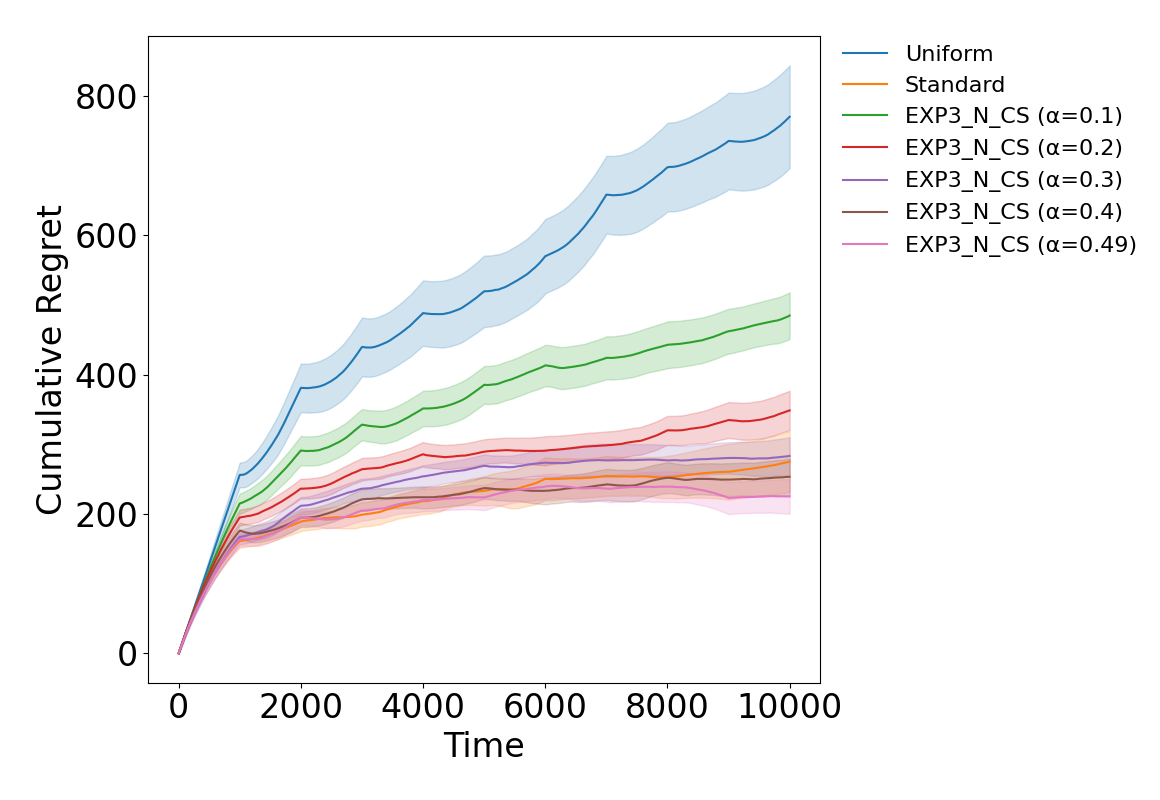}\label{fig4b}}\\
         \subfigure[CS width]{
\includegraphics[width=0.42\linewidth]{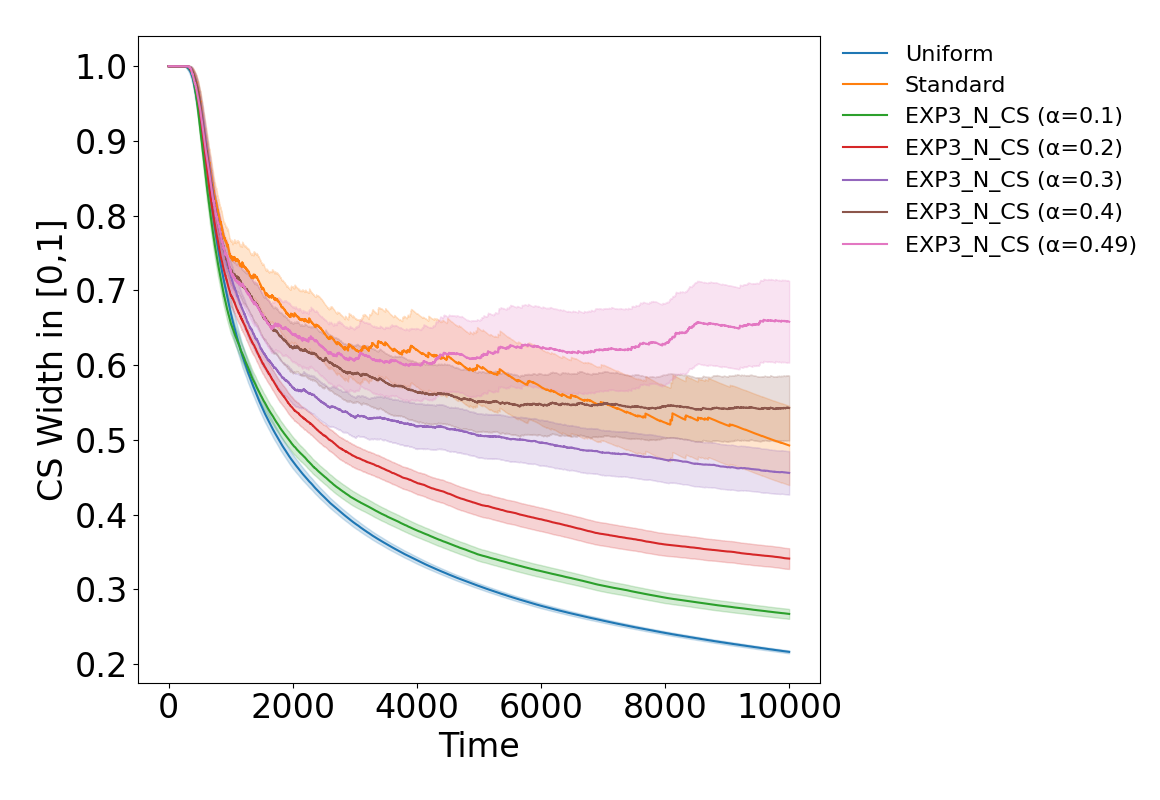}\label{fig4c}}
\subfigure[Maximum ATE estimation error]{
\includegraphics[width=0.45\linewidth]{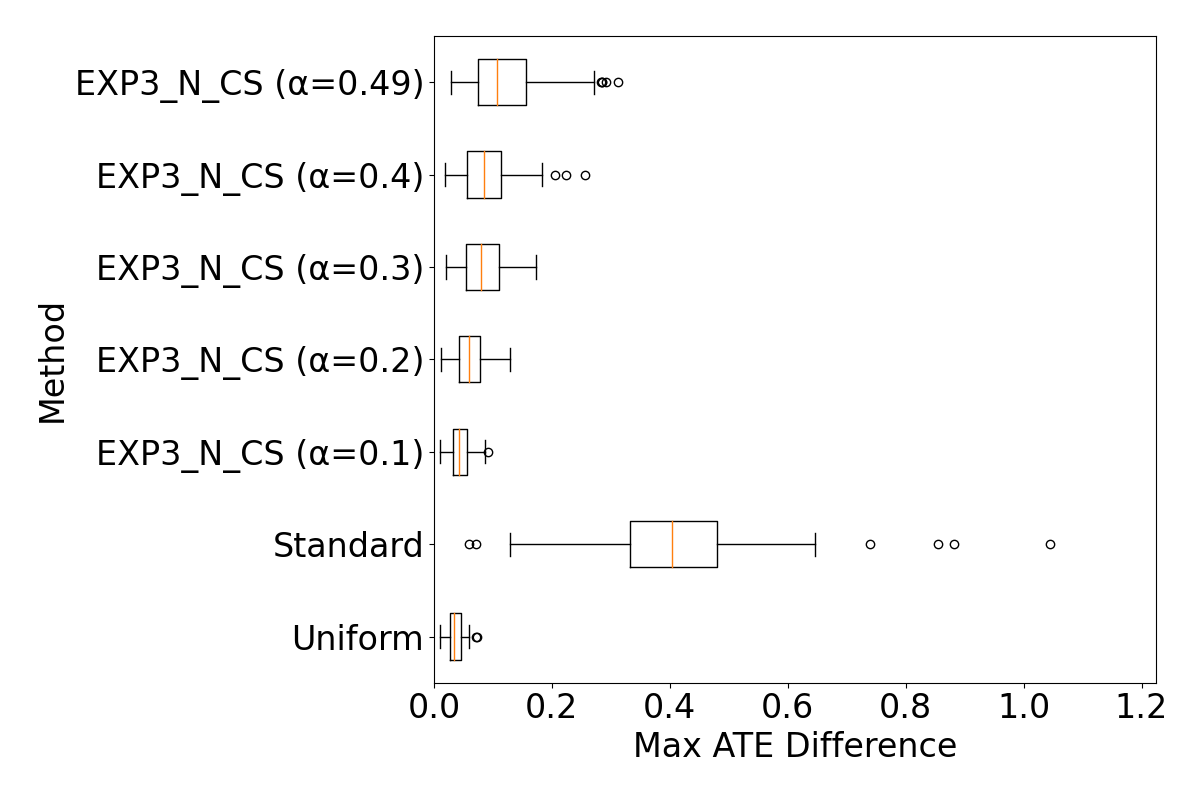}\label{fig4d}}
\caption{Experimental results of instance 2.}\label{fig4}
\end{figure}

\begin{figure}[t]
\centering
 \subfigbottomskip=4pt
	\subfigcapskip=-5pt 
	\subfigure[Network topology]{
\includegraphics[width=0.28\linewidth]{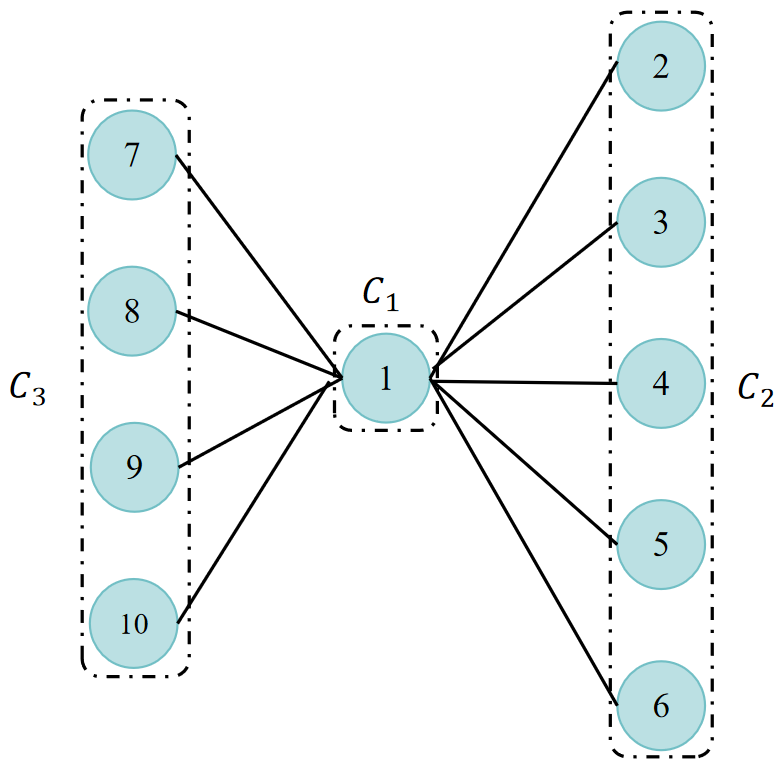}\label{fig5a}}\quad\quad
	\subfigure[Cumulative regret]{
\includegraphics[width=0.42\linewidth]{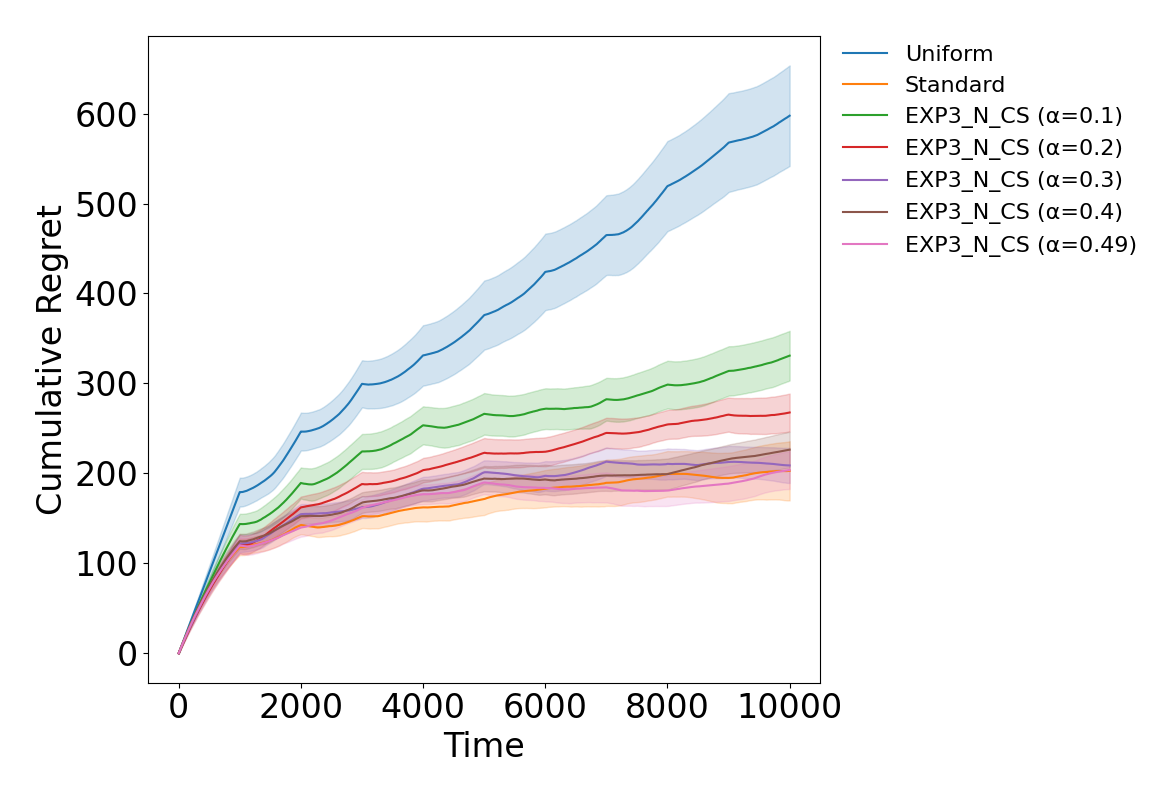}\label{fig5b}}\\
         \subfigure[CS width]{
\includegraphics[width=0.42\linewidth]{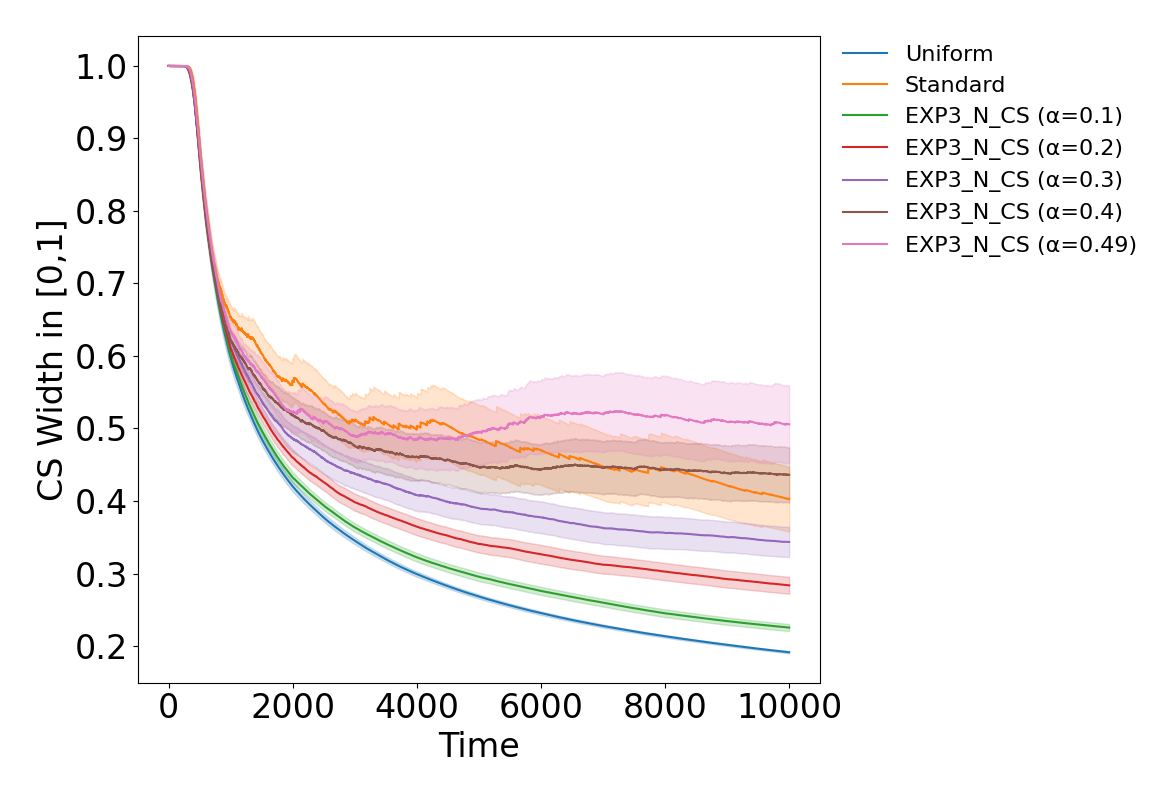}\label{fig5c}}
\subfigure[Maximum ATE estimation error]{
\includegraphics[width=0.45\linewidth]{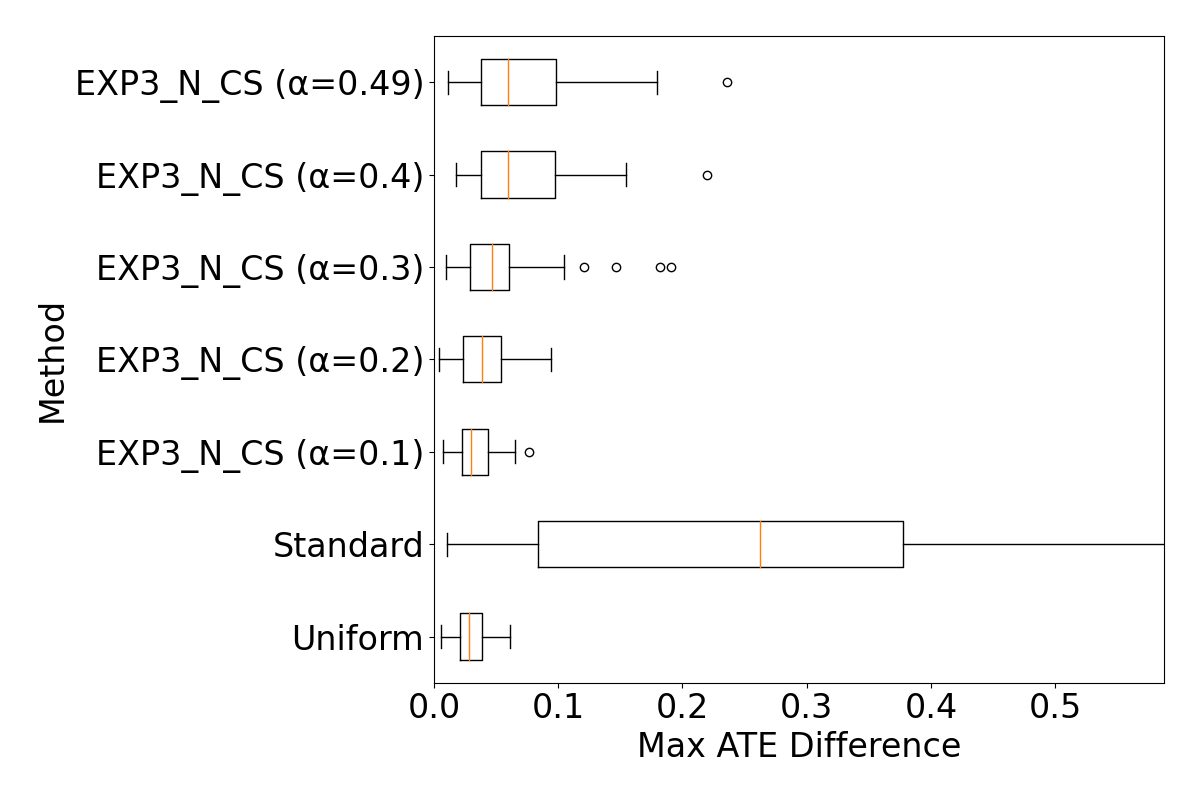}\label{fig5d}}
\caption{Experimental results of instance 3.}\label{fig5}
\end{figure}

\begin{figure}[t]
\centering
 \subfigbottomskip=4pt
	\subfigcapskip=-5pt 
	\subfigure[Network topology]{
\includegraphics[width=0.28\linewidth]{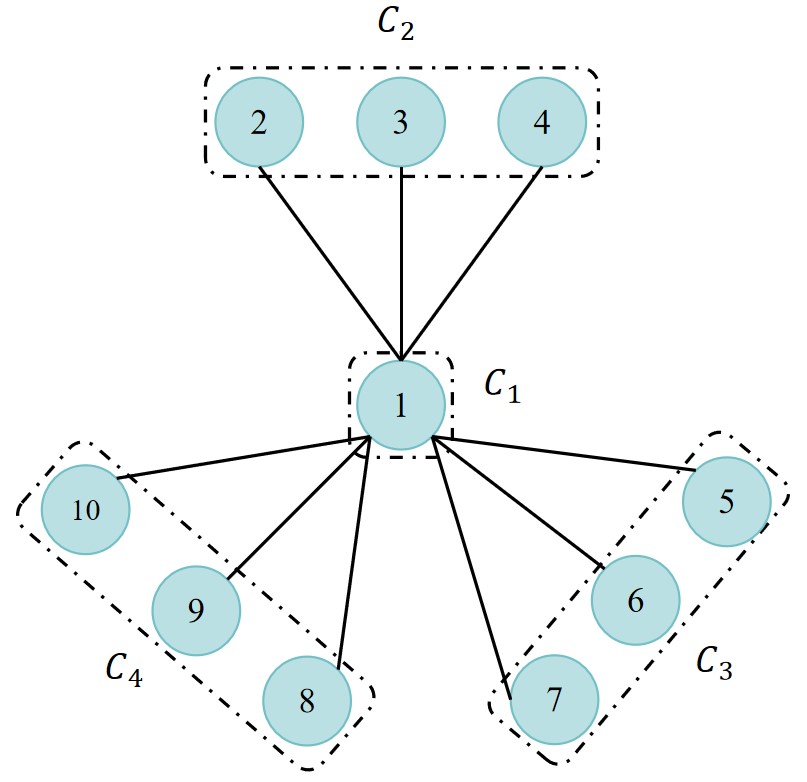}\label{fig6a}}\quad\quad
\subfigure[Cumulative regret]{
\includegraphics[width=0.42\linewidth]{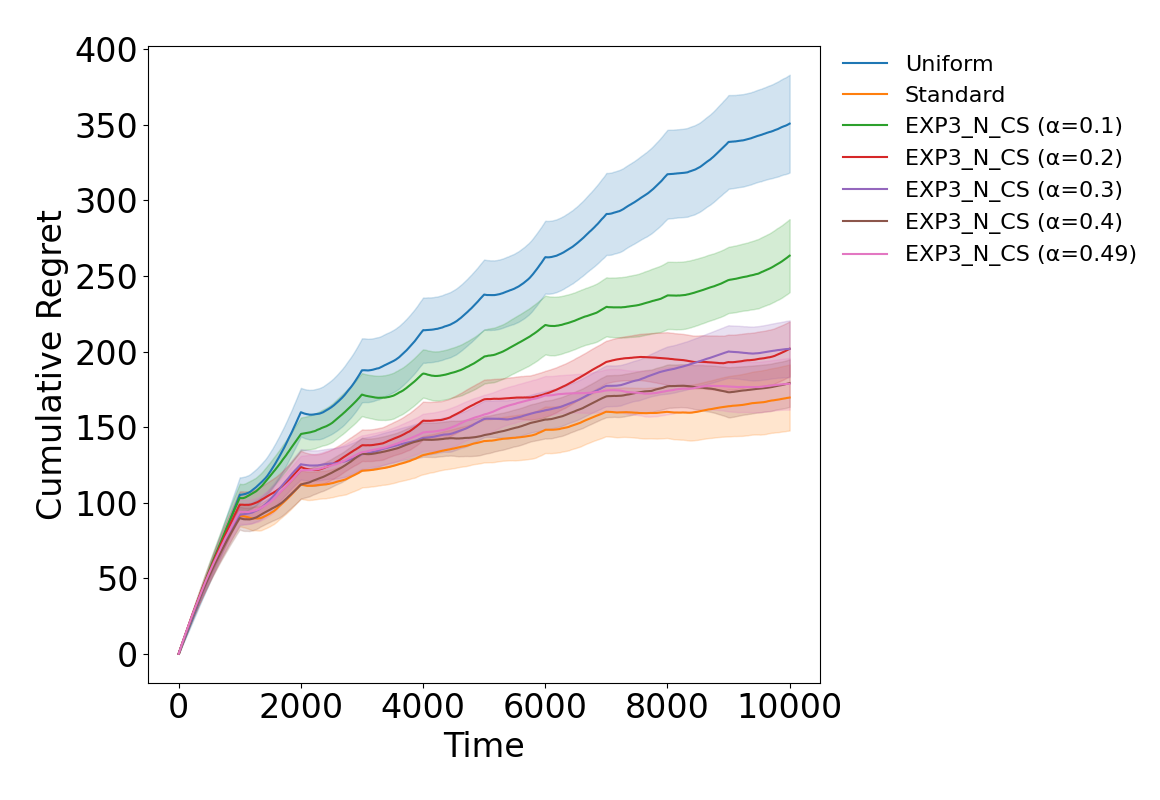}\label{fig6b}}\\
         \subfigure[CS width]{
\includegraphics[width=0.42\linewidth]{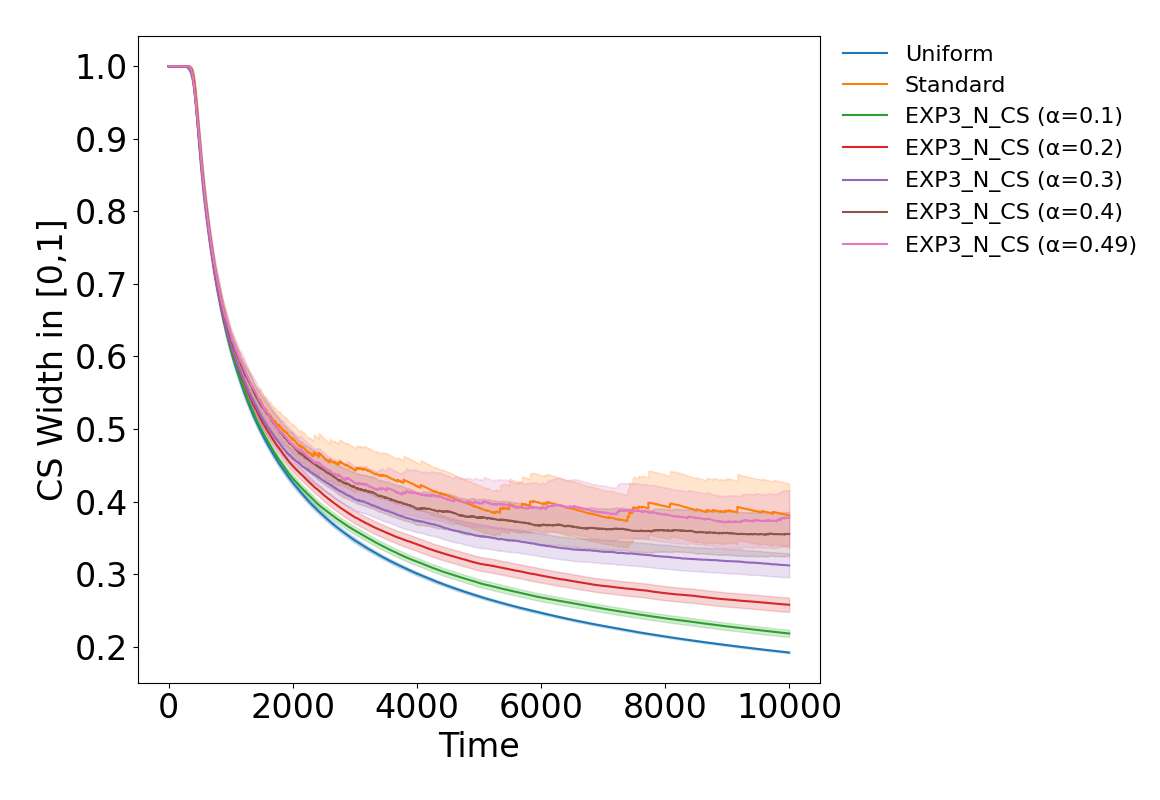}\label{fig6c}}\quad
\subfigure[Maximum ATE estimation error]{
\includegraphics[width=0.45\linewidth]{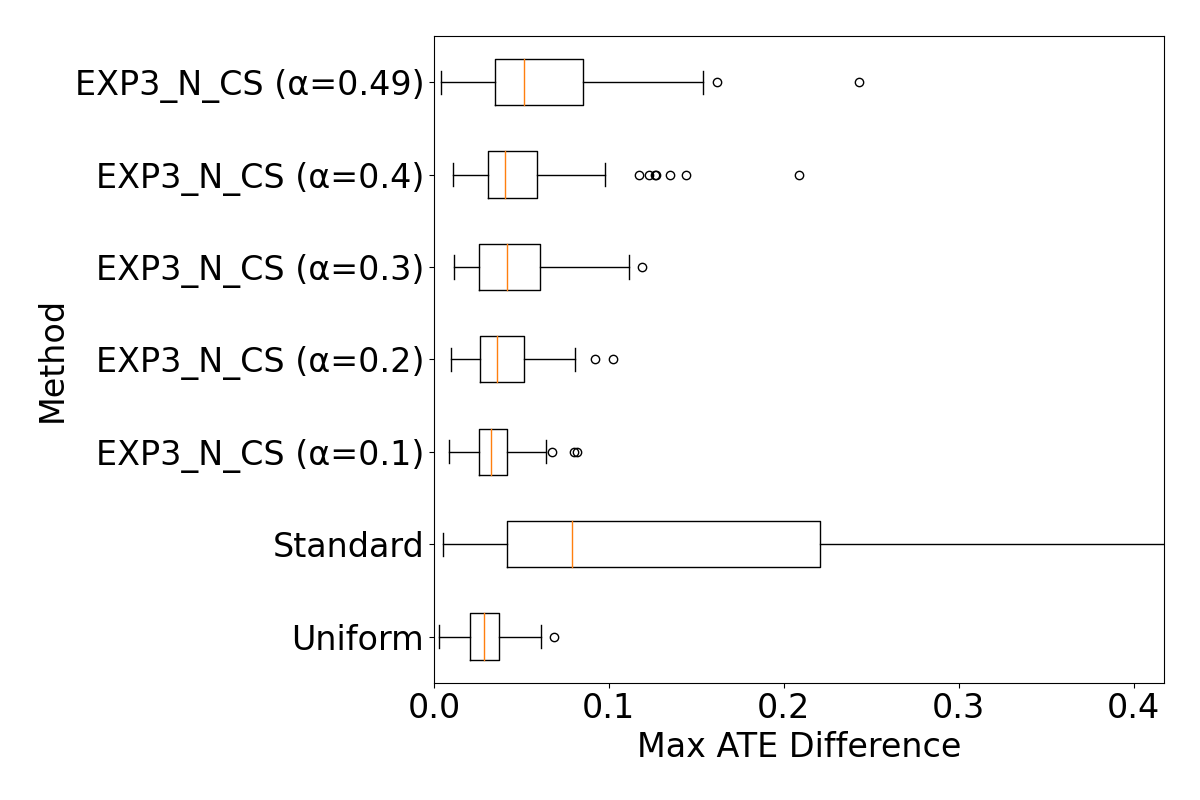}\label{fig6d}}
\caption{Experimental results of instance 4.}\label{fig6}
\end{figure}

\section{Additional Experimental Results}\label{AExperiment}

In this section, we present four additional experiment instances along with the corresponding results.

\paragraph{Instance 1: single unit.} In this setup, the network consists of a single unit (as illustrated in Fig. \ref{fig3a}), making it identical to the case considered in \citet{liang2023experimental}. Additionally, the action set is defined to include five actions, i.e., $\A = \{0, \dots, 4\}$. The reward structure and baseline algorithms are configured in the same manner as described in Section \ref{experiments}.

\paragraph{Instance 2: 6 units.} 
The network in this setup consists of 6 units, organized in a loop topology, as illustrated in Fig. \ref{fig4a}. Furthermore, the network is divided into three clusters, with the cluster structure also depicted in Fig. \ref{fig4a}. The configuration of the action set, reward structure, and baseline algorithms follows the same setup as described in Section \ref{experiments}.

\paragraph{Instance 3: 10 units case 1} 
In this setup, the network consists of 10 units, with the topology structure depicted in Fig. \ref{fig5a}. Additionally, we divide the network into three clusters, and the cluster structure is also illustrated in Fig. \ref{fig5a}. The configuration of the action set, reward structure, and baseline algorithms remains the same as described in Section \ref{experiments}.

\paragraph{Instance 4: 10 units case 2}
The network consists of 10 units arranged in a star-like topology, as shown in Figure~\ref{fig6a}. At the center of the network lies a single unit forming the central cluster, which is directly connected to every unit in three outer clusters. Each of these outer clusters comprises 3 units, resulting in a total of 9 peripheral units connected to the central cluster. The configuration of the action set, reward structure, and baseline algorithms remains the same as described in Section \ref{experiments}.

We ran the algorithms 100 times and reported the average results.

The experimental results are presented in Fig. \ref{fig3}, \ref{fig4}, \ref{fig5} and \ref{fig6}. For cumulative regret, both the Standard approach and \texttt{EXP3-N-CS} with larger $\alpha$ values achieve the lowest regret, while the Uniform baseline incurs the highest regret. 
For continual inference, although Standard exhibits narrower CS widths than some \texttt{EXP3-N-CS} variants, its intervals are invalid due to the lack of theoretical guarantees.
This issue is reflected in the maximum ATE estimation error, where Standard exhibits the largest errors with many outliers. 
In contrast, \texttt{EXP3-N-CS} with moderate or large $\alpha$ and the \textit{Uniform} baseline achieve lower estimation errors.

\end{document}